\documentclass[runningheads]{llncs}
\usepackage[T1]{fontenc}
\usepackage{tikz}
\usepackage{eso-pic}

\usepackage[utf8]{inputenc}

\usepackage[sc]{mathpazo}

\usepackage{tikz}
\usetikzlibrary{positioning, arrows.meta}

\usepackage{amsmath,amssymb,amsfonts,mathrsfs,bm}

\usepackage{graphicx}

\usepackage{xcolor}

\usepackage{mathtools}

\usepackage{microtype}

\usepackage{mleftright}

\usepackage{xspace}

\usepackage[ruled,vlined,linesnumbered]{algorithm2e}

\makeatletter
\AddToHook{cmd/appendix/before}{\def\cref@section@alias{appendix}}
\makeatother

\usepackage{enumitem}

\newcommand{\N}{\ensuremath{\mathbb{N}}}

\newcommand{\R}{\ensuremath{\mathbb{R}}}
\newcommand{\Z}{\ensuremath{\mathbb{Z}}}

\DeclarePairedDelimiter\abs{\lvert}{\rvert}
\DeclarePairedDelimiter\norm{\lVert}{\rVert}

\DeclarePairedDelimiter{\Paren}{(}{)}
\DeclarePairedDelimiter{\Bracket}{[}{]}
\DeclarePairedDelimiter{\Brace}{\{}{\}}

\NewDocumentCommand{\Oh}{som}{%
  \ensuremath{\mathcal{O}\IfBooleanTF{#1}{%
    \Paren*{#3}
  }{%
    \IfNoValueTF{#2}{%
      \Paren{#3}
    }{%
      \Paren[#2]{#3}
    }%
  }%
}}

\NewDocumentCommand{\Om}{som}{%
\ensuremath{\Omega\IfBooleanTF{#1}{%
    \Paren*{#3}%
  }{%
    \IfNoValueTF{#2}{%
      \Paren{#3}%
    }{%
      \Paren[#2]{#3}%
    }%
  }%
}}

\NewDocumentCommand{\Th}{som}{%
\ensuremath{\Theta\IfBooleanTF{#1}{%
    \Paren*{#3}%
  }{%
    \IfNoValueTF{#2}{%
      \Paren{#3}%
    }{%
      \Paren[#2]{#3}%
    }%
  }%
}}

\NewDocumentCommand{\oh}{som}{%
\ensuremath{o\IfBooleanTF{#1}{%
    \Paren*{#3}%
  }{%
    \IfNoValueTF{#2}{%
      \Paren{#3}%
    }{%
      \Paren[#2]{#3}%
    }%
  }%
}}

\NewDocumentCommand{\om}{som}{%
\ensuremath{\omega\IfBooleanTF{#1}{%
    \Paren*{#3}%
  }{%
    \IfNoValueTF{#2}{%
      \Paren{#3}%
    }{%
      \Paren[#2]{#3}%
    }%
  }%
}}

\NewDocumentCommand{\Prob}{som}{%
\ensuremath{\mathbb{P}
  \IfBooleanTF{#1}{%
    \Bracket*{#3}%
  }{%
    \IfNoValueTF{#2}{%
      \Bracket{#3}%
    }{%
      \Bracket[#2]{#3}%
    }%
  }%
}}

\NewDocumentCommand{\E}{som}{%
\ensuremath{\mathbb{E}
  \IfBooleanTF{#1}{%
    \Bracket*{#3}%
  }{%
    \IfNoValueTF{#2}{%
      \Bracket{#3}%
    }{%
      \Bracket[#2]{#3}%
    }%
  }%
}}

\NewDocumentCommand{\Max}{s o m m}{%
\ensuremath{\max
  \IfBooleanTF{#1}{%
    \Brace*{#3,\, #4}%
  }{%
    \IfNoValueTF{#2}{%
      \Brace{#3,\, #4}%
    }{%
      \Brace[#2]{#3,\, #4}
    }%
  }%
}}

\NewDocumentCommand{\Min}{s o m m}{%
\ensuremath{\min
  \IfBooleanTF{#1}{%
    \Brace*{#3,\, #4}%
  }{%
    \IfNoValueTF{#2}{%
      \Brace{#3,\, #4}%
    }{%
      \Brace[#2]{#3,\, #4}
    }%
  }%
}}

\NewDocumentCommand{\Exp}{som}{%
\ensuremath{\exp
  \IfBooleanTF{#1}{%
    \Paren*{#3}%
  }{%
    \IfNoValueTF{#2}{%
      \Paren{#3}%
    }{%
      \Paren[#2]{#3}%
    }%
  }%
}}

\NewDocumentCommand{\Log}{som}{%
\ensuremath{\log
  \IfBooleanTF{#1}{%
    \Paren*{#3}%
  }{%
    \IfNoValueTF{#2}{%
      \Paren{#3}%
    }{%
      \Paren[#2]{#3}%
    }%
  }%
}}

\NewDocumentCommand{\Var}{som}{%
\ensuremath{\mathrm{Var}
  \IfBooleanTF{#1}{%
    \Bracket*{#3}%
  }{%
    \IfNoValueTF{#2}{%
      \Bracket{#3}%
    }{%
      \Bracket[#2]{#3}%
    }%
  }%
}}

\NewDocumentCommand{\SetBuilder}{s o m m}{%
\ensuremath{\IfBooleanTF{#1}{%
    \Brace*{\,#3 : #4\,}%
  }{%
    \IfNoValueTF{#2}{%
      \Brace{\,#3 : #4\,}%
    }{%
      \Brace[#2]{\,#3 : #4\,}%
    }%
  }%
}}

\newcommand{\onemax}{\textsc{OneMax}\xspace}

\newcommand{\eps}{\ensuremath\varepsilon}

\renewcommand{\phi}{\ensuremath{\varphi}}

\newcommand{\opoEA}{\ensuremath{(1+1)\text{-EA}}\xspace}

\newcommand{\Bin}{\ensuremath{\mathrm{Bin}}}
\newcommand{\Hyp}{\ensuremath{\mathrm{Hyp}}}
\newcommand{\Ber}{\ensuremath{\mathrm{Bernoulli}}}
\newcommand{\Poi}{\ensuremath{\mathrm{Poi}}}
\newcommand{\Unif}{\ensuremath{\mathrm{Unif}}}

\newcommand{\ind}{\ensuremath{\mathbf{1}}}

\newcommand{\totVar}[1]{\ensuremath{\norm{#1}_{\mathrm{TV}}}}

\usepackage{hyperref}
\usepackage{color}

\usepackage{cleveref}
\begin{document}
\title{The $\opoEA$ in Dynamic Environments}
%
%

\author{Georg Hasebe\orcidID{0009-0002-9643-1712} \and\\ Johannes Lengler\orcidID{0000-0003-0004-7629} \and\\
Raghu Raman Ravi\orcidID{0000-0003-3641-1824}}
\authorrunning{G. Hasebe, J. Lengler, R. Ravi}
%
\institute{Eidgen\"{o}ssische Technische Hochschule Z\"{u}rich, Switzerland \email{\{georg.hasebe,johannes.lengler,raghu.ravi\}@inf.ethz.ch}.}
\maketitle              
\begin{abstract}
We study the $\opoEA$ in dynamic linear environments, where in every generation selection is performed with respect to a freshly sampled linear function with positive weights.
We consider the Dynamic Binary Value problem, where each generation uses a uniformly random permutation of $1,2,4,\dots,2^{n-1}$, and a Uniform weight variant, where the weights are drawn independently from $\Unif(0,1)$. Both of them have recently been integrated into the IOHprofiler platform and empirically studied. 

For both models we prove a sharp threshold in the mutation parameter $\chi$ for mutation rate $\chi/n$.
Below the threshold, the expected optimisation time is $\Oh{n\log n}$, whereas above it the runtime becomes $2^{\Om{n}}$.

For the Dynamic Binary Value problem in the exponential regime, we also quantify at what distance from the optimum the optimisation process stagnates. We show that there is a second threshold: a distance that is efficiently reached, but reaching any smaller distance takes exponential time. This quantifies and proves previous empirical findings.
\end{abstract}

\section{Introduction}
\label{sec:introduction}

Evolutionary algorithms (EAs) are widely used as general-purpose heuristics for black-box optimisation, where the objective function is accessible only through function evaluations.
A central question in both theory and practice is how algorithmic design choices like the mutation rate translate into optimisation performance.
While mutation rate effects on static benchmark functions such as \onemax or linear functions are well understood~\cite{witt2013tight}, the picture becomes more subtle in dynamic environments, i.e., settings in which the fitness function changes over time.

In order to study this question, 
Vermetten et al.~\cite{vermettenEmpiricalAnalysisDynamic2024} have recently integrated several dynamic benchmarks into the IOHprofiler framework~\cite{doerrBenchmarkingDiscreteOptimization2020} and performed large-scale benchmarking experiments on those environments. The benchmarks are a generalized version of dynamic linear functions as introduced in~\cite{lengler1+1EANoisyLinear2018}.\footnote{They were called ``noisy linear functions'' in~\cite{lengler1+1EANoisyLinear2018}.} 
At any fixed point in time, selection is based on a pseudo-Boolean linear function with positive weights. Thus in particular, the global optimum of the search space $\{0,1\}^n$ remains stable at the all-one string, there are no other local optima, and at each point in time the fitness landscape is monotone, meaning that flipping a zero-bit into a one-bit improves fitness. Selection is always comparison-based with respect to the current fitness function. Two central instances studied in~\cite{vermettenEmpiricalAnalysisDynamic2024} are the Dynamic Binary Value (DBV) benchmark introduced in~\cite{lenglerMeier2024} where the weights $2^0,2^1,\ldots,2^{n-1}$ are redistributed among the $n$ bits in each round; and the dynamic linear function where weights are redrawn randomly in each generation from the uniform distribution $\Unif(0,1)$.

In parts, the simulation results in~\cite{vermettenEmpiricalAnalysisDynamic2024} confirmed previous theoretical predictions and experiments~\cite{lenglerMeier2024,lengler1+1EANoisyLinear2018,lenglerRiedi2022}. In particular,~\cite{lengler1+1EANoisyLinear2018} had predicted for some dynamic linear functions that the performance of the $\opoEA$ should drop dramatically with the mutation rate for many dynamic linear functions, and this had been theoretically extended to DBV and experimentally confirmed in~\cite{lenglerMeier2024}. 
However, \cite{vermettenEmpiricalAnalysisDynamic2024} also raised new questions. For mutation rates above the critical rate, it was observed that the fraction of correct bits quickly reaches a plateau at which progress comes to a complete halt. 
For DBV with mutation rate $3/n$, the $\opoEA$ remains below $80\%$ correct bits even after the full budget of $1000n$ generations had been exhausted. It was noted in~\cite{vermettenEmpiricalAnalysisDynamic2024} that the distance from the optimum is surprisingly large, and the location of the plateau had not been quantified by theory before. The first novel contribution of this paper is an explicit formula for the location of this plateau. Moreover, we prove that indeed, while this plateau is reached in time $\Oh{n\log n}$ by the $\opoEA$, any further progress towards the global optimum takes exponential time.

For the other benchmark, where weights are randomly redrawn from $\Unif(0,1)$,~\cite{vermettenEmpiricalAnalysisDynamic2024} was the first paper to empirically study this distribution. As for DBV, they observed a critical mutation rate above which the optimisation time increases dramatically. Our second contribution is to pinpoint the exact critical rate and prove rigorously that the optimisation time switches from time $\Oh{n \log n}$ to exponential at this rate. 

\subsection*{Our Contribution in Detail}
We will now discuss our results in more detail. The following discussion is for the $\opoEA$ with standard bit mutation and mutation rate $\chi/n$. 

\paragraph{A mistake in previous proofs.} Lengler and Schaller~\cite{lengler1+1EANoisyLinear2018} have identified a sharp threshold $\chi_0\approx 1.59362$ for the $\opoEA$ on the set of dynamic linear functions. More precisely, for every fixed $0<\chi<\chi_0$ the $\opoEA$ optimises any dynamic linear function in $\Oh{n\log n}$ steps whereas for every fixed $\chi>\chi_0$ there are dynamic linear functions for which the optimisation time is superpolynomial. It was observed in~\cite{lenglerMeier2024} that the proof implies the same runtime results for DBV, which is why we will call the constant $\chi_0 = \chi_\mathrm{dbv}$ henceforth. However, as we will argue in Section~\ref{sec:erratum}, the proof in~\cite{lengler1+1EANoisyLinear2018} contained a mistake, and thus the results in~\cite{lengler1+1EANoisyLinear2018} and~\cite{lenglerMeier2024} only hold under the additional assumption that the optimisation is started sufficiently close to the optimum. Our proof fixes this issue and recovers the results from ~\cite{lengler1+1EANoisyLinear2018} and~\cite{lenglerMeier2024} as they were originally claimed. As we will see, the gap in the proof was substantial, and fixing it requires a non-trivial argument about convexity of the drift function. We also strengthen the original statements in several aspects: we show that the runtime in the negative cases is $2^{\Omega(n)}$ instead of just superpolynomial. Moreover, we give a unified proof for the efficient regime for a larger class of permutation-invariant dynamic environments that contains both dynamic linear functions and DBV. The main result is \Cref{thm:add:main} in~\Cref{sec:addendum}.

\paragraph{DBV: Locating the Plateau.} 
For DBV, other than previous work we can also pinpoint \emph{where} the algorithm gets stuck for large mutation rates. Once $\chi>\chi_\mathrm{dbv}$, 
we identify a constant $\alpha^*(\chi)\in(0,1/2)$ such that the $\opoEA$ reaches distance $\alpha n$ from the optimum for any $\alpha > \alpha^*(\chi)$ in time $\Oh{n\log n}$, but reaching distance $\alpha n$ for any $\alpha<\alpha^*(\chi)$ takes exponential time. Thus we derive the exact location of the plateau at which the algorithm gets stuck. Our formula for $\alpha^*(\chi)$ is explicit, Equation~\eqref{eq:dbv:alpha_star_cont} in \Cref{sec:dbv}. For instance, for $\chi=3$ the resulting bound corresponds to a stable plateau around a ratio of $1-\alpha^*(3)\approx 73\%$ correct bits. This confirms the empirical reports in \cite[Figure~3]{vermettenEmpiricalAnalysisDynamic2024} that progress of the $\opoEA$ stops at a surprisingly large distance from the optimum, even for rather moderate values of $\chi$. Those results are contained in ~\Cref{thm:dbv:main} in ~\Cref{sec:dbv}. Finally, in \Cref{sec:dbv:finite_n} we also quantify the finite-size effects by quantifying how much the location of the plateau shifts for finite $n$.

\paragraph{Uniform Weights.}
For the Uniform weight model with weights in $\Unif(0,1)$ we obtain a substantially larger critical value $\chi_\mathrm{unif}\approx 2.76531$ as the unique positive root arising in our drift analysis in Theorem~\ref{thm:unif:main}.
For every fixed $0<\chi<\chi_\mathrm{unif}$ the runtime is $\Oh{n\log n}$, whereas for every fixed $\chi>\chi_\mathrm{unif}$ it is $2^{\Om{n}}$. 
In particular, this identifies an intermediate regime $\chi_\mathrm{dbv}<\chi<\chi_\mathrm{unif}$ in which DBV already exhibits the pronounced slowdown, while the Uniform weight model is still efficiently optimisable.

\subsection*{Related Work}
\paragraph{Mutation rate on linear and monotone functions.}
In the static setting, the runtime of the $\opoEA$ on linear pseudo-Boolean
functions is well understood.
Droste, Jansen and Wegener~\cite{drosteJansenWegener2002linear} proved that
every linear function is optimised in expected $\Oh{n\log n}$ time with
mutation rate~$1/n$.
Using multiplicative drift
analysis~\cite{doerrMultiplicativeDriftAnalysis2012},
Witt~\cite{witt2013tight} sharpened this to the tight bound
$(1+\oh{1})\,e\,n\ln n$ and extended the analysis to arbitrary mutation rates
$\chi/n$, showing that every constant~$\chi$ still yields polynomial runtime.
Thus on static linear functions the mutation parameter has no qualitative
impact on the runtime order.

In other settings, the mutation rate is critical. This showed in particular in a line of work studying monotone functions, where
flipping a zero-bit into a one-bit always increases fitness.
Doerr, Jansen, Sudholt, Winzen and Zarges~\cite{doerrJSSWZ2013} proved that
for~$\chi<1$ the $\opoEA$ optimises every monotone function in $\Oh{n\log n}$
time, while for sufficiently large constant~$\chi$ some monotone functions require
exponential time.
Lengler~\cite{lengler2020dichotomy} introduced the \textsc{HotTopic}
construction as a specifically hard monotone function. He sharpened the upper threshold to $\chi > 2.13$ and extended
the dichotomy to several algorithm variants including the $(\mu+1)$-EA and
$(\mu+1)$-GA. For the latter, he found that a larger population size allows for arbitrarily large values of $\chi$ if the search is started sufficiently close to the optimum. Conversely, Lengler and Zou~\cite{lenglerZou2021} showed that without crossover, sufficiently large populations cause exponential slowdown for the $(\mu+1)$-EA 
even for arbitrarily small~$\chi$, so the effect of populations is two-sided, and our picture is still very incomplete. 
In the other direction for general monotone functions, Lengler, Martinsson and
Steger~\cite{lenglerMartinssonSteger2019} showed via an entropy compression
argument that the efficient regime extends slightly beyond~$\chi=1$. The critical value for the phase transition on monotone functions remains open.

\paragraph{Dynamic environments.}
There is a long tradition of studying evolutionary algorithms in dynamic
environments; see~\cite{branke2002evolutionary,nguyenYangBranke2012} for
surveys.
The first rigorous runtime analysis is due to
Droste~\cite{droste2002analysis,droste2003analysis}, who analysed the
$\opoEA$ on dynamic \onemax, where the target string moves over time.
The resulting runtime was polynomial if the target moved at most
$\Oh{\log n/n}$ bits per generation and superpolynomial for
$\om{\log n/n}$ bits.

This started a systematic investigation of which evolutionary algorithms can
track a moving optimum.
K\"otzing, Lissovoi and Witt~\cite{kotzing2015one} showed in an anytime
analysis that even for larger mutation rates the $\opoEA$ stays close to the
optimum of dynamic \onemax and
generalisations.
Various mechanisms have been shown to aid tracking: larger population sizes~\cite{dang2017populations} and
offspring population sizes~\cite{jansen2005theoretical}; diversity
mechanisms~\cite{oliveto2013analysis}; and island
models~\cite{LassigSudholt2010}.
Alternative algorithmic paradigms such as Ant Colony
Optimisation~\cite{kotzing2012aco,lissovoi2015runtime}, Artificial Immune
Systems~\cite{jansen2014evolutionary}, and Evolutionary
Strategies~\cite{arnold2006optimum} have also been studied in dynamic
settings.
For other base functions, tracking the optimum can be considerably more
difficult~\cite{rohlfshagen2009dynamic,jansen2005theoretical}.

\paragraph{Dynamic linear functions and Dynamic BinVal.}

The model of dynamic linear functions was introduced by Lengler and Schaller~\cite{lengler1+1EANoisyLinear2018} and extended to DBV as a ``limiting model'' in \cite{lenglerMeier2024}. 
In this model, in each generation $t$ a set of weights $W_1^{(t)},\ldots, W_n^{(t)}$ is drawn {independently and identically distributed (i.i.d.) from some distribution~$\mathcal D$ with positive support, and the fitness function used for selection in generation $t$ is then the linear function
\begin{equation*}
    f^{(t)}(x)=\sum_{i=1}^n W_i^{(t)} x_i.
\end{equation*}

As mentioned before,~\cite{lengler1+1EANoisyLinear2018} proved a sharp performance 
threshold in the mutation rate $\chi/n$ at~$\chi_{\textrm{DBV}} \approx 1.59$, modulo the gap in the proof detailed in Section~\ref{sec:erratum}. 
Subsequent work has explored extensions beyond the $\opoEA$ on specific weight
distributions.
Lengler and Riedi~\cite{lenglerRiedi2022} analysed the $(\mu+1)$-EA on
DBV, showing that moderate population sizes increase the effective
threshold near the optimum and, surprisingly, that the hardest region of
optimisation lies far from the optimum rather than close to it.
Lengler and Meier~\cite{lenglerMeier2024} complemented this with experiments
indicating that crossover substantially extends the range of efficient
mutation parameters.
Kaufmann et al.~\cite{kaufmann2025hardest}
introduced Switching Dynamic BinVal (SDBV), proved it is drift-minimising
among dynamic monotone functions for any mutation rate, and showed that the
$\opoEA$ optimises it in $\Theta(n^{3/2})$ generations.
Janett and Lengler~\cite{janett2023twodimensional} applied a
two-dimensional drift framework to \textsc{TwoLin}, a minimal dynamic
environment where only two linear functions alternate rather than a new
permutation being drawn each generation, establishing that the threshold
phenomenon persists beyond the DBV model.
On the empirical side, Vermetten~et~al.\
\cite{vermettenEmpiricalAnalysisDynamic2024} integrated dynamic linear functions and several DBV variants into the IOHprofiler framework and performed large-scale
benchmarking, observing the mutation-rate sensitivity and benchmark
divergence that motivate the present work.

\section{Setup and Terminology}
\label{sec:setup}

\paragraph{Dynamic Linear Functions.}
We consider a dynamic optimisation setting in which, in every generation $t\in\N_0$, the environment provides a random linear fitness function
\begin{equation*}
    f^{(t)}(x)\coloneqq \sum_{i=1}^n W_i^{(t)} x_i,
    \quad x=(x_1,\dots,x_n)\in\{0,1\}^n,
\end{equation*}
where $W^{(t)}=(W_1^{(t)},\dots,W_n^{(t)})\in\R_{> 0}^n$ is a random weight vector. Since the weights are positive, the bit string $(1,\dots,1)$ maximises every $f^{(t)}$.

In the model of \emph{dynamic linear functions}~\cite{lenglerMeier2024}, $(W^{(t)})_{t\ge 0}$ is an i.i.d.\ sequence of weight vectors with a fixed distribution $\mathcal{D}$ on $\R_{> 0}^n$.
The optimisation algorithm compares fitness values under the current function $f^{(t)}$. The sampling of $W^{(t)}$ (and thus of $f^{(t)}$) is part of the dynamic environment, not the algorithm.

\paragraph{Weight Models.}
In this paper we restrict attention to two concrete instances of the above model:
\begin{itemize}[leftmargin=*]
    \item \textbf{DBV:} $W^{(t)}$ is a uniformly random permutation of $(2^0,2^1,\dots,2^{n-1})$. Note that this does not fall in the category of dynamic linear functions, as the weights are not independent.
    \item \textbf{Uniform:} $W_1^{(t)},\dots,W_n^{(t)}$ are i.i.d.\ $\Unif(0,1)$.
\end{itemize}

\paragraph{The $\opoEA$.}
We study the $\opoEA$ with standard bit mutation and mutation rate $\chi/n$ where selection is based on the dynamic fitness environment $(f^{(t)})_{t\ge 0}$ as shown in \Cref{alg:pre:opoea}.

\begin{algorithm}[h!]
choose $x^{(0)}\in\{0,1\}^n$ u.a.r.\;
\For{$t=0,1,\dots$}
{
create $y^{(t)}\in\{0,1\}^n$ by flipping each bit of $x^{(t)}$ independently with probability $\chi/n$\;
\uIf{$f^{(t)}(y^{(t)})\ge f^{(t)}(x^{(t)})$} {
    $x^{(t + 1)}\leftarrow y^{(t)}$\;
} \Else {
    $x^{(t + 1)}\leftarrow x^{(t)}$\;
}
}
\caption{The $\opoEA$ with mutation rate $\chi/n$ maximising the dynamic environment $(f^{(t)})_{t\ge 0}$.}
\label{alg:pre:opoea}
\end{algorithm}

\paragraph{Runtime and Drift Analysis.}
Since all weights are non-negative, the global optimum is $(1,\dots,1)$.
We define the runtime as
\begin{equation*}
    T \coloneqq \min\SetBuilder{t\ge 0}{x^{(t)}=(1,\dots,1)}.
\end{equation*}
Let $Y_t$ denote the number of zero-bits in $x^{(t)}$.
Then $T=\min\SetBuilder{t\ge 0}{Y_t=0}$.
Moreover, by symmetry of standard bit mutation and the i.i.d.\ resampling of $(W^{(t)})_{t\ge 0}$, the conditional distribution of $Y_{t+1}\mid Y_t=y$ depends only on $y$, hence $(Y_t)_{t\ge 0}$ is a time-homogeneous Markov chain.

Our main tool is drift analysis for the Markov chain $(Y_t)_{t\ge 0}$.
For $y\in\{0,1,\dots,n\}$ we define the drift at state $y$ as
\begin{equation}
\label{eq:pre:standard_drift}
    \Delta(y)\coloneqq \E{Y_t-Y_{t+1}\mid Y_t=y}.
\end{equation}

\paragraph{Notation.}
We write $\N_0\coloneqq\{0,1,2,\dots\}$ and $[n]\coloneqq\{1,\dots,n\}$.
Throughout the paper, we consider the $\opoEA$ with mutation parameter $\chi>0$ and mutation rate $\chi/n$, and the search point in generation $t$ is denoted by $x^{(t)}\in\{0,1\}^n$. We denote by $Y_t$ its number of zero-bits.
For $Y_t =y\in\{0,1,\dots,n\}$ we set $\alpha\coloneqq y/n\in[0,1]$ as the fraction of zero-bits. We write $X\overset{d}{=}Y$ to denote equality in distribution.

We say that a sequence of events $(A_n)_{n\ge 1}$ holds asymptotically almost surely (a.a.s.) if $\Prob{A_n}=1-\oh{1}$ as $n\to\infty$.
Unless stated otherwise, all asymptotic notation (e.g.\ $\Oh{\cdot}$, $\Om{\cdot}$, $\oh{\cdot}$) refers to the limit $n\to\infty$.
Hidden constants in $\Oh{\cdot}$ etc.\ may depend on $\chi$ and on the weight model, but never on $n$.
When we use asymptotic notation in $\alpha$ (e.g.\ $\Oh{\alpha^2}$), it refers to the limit $\alpha\to 0^+$ for fixed~$\chi$.

\section{Preliminaries}
\label{sec:preliminaries}

In this section we will collect some preparatory observations about the $\opoEA$ that hold in any dynamic environment where the fitness function in each generation is a linear function with positive weights.

\subsection{Conditioning on the Mutation Outcome}

Let $K$ denote the number of flipped bits in one mutation step.
Since each bit flips independently with probability $\chi/n$, we have
\begin{equation*}
    K\sim \Bin(n,\chi/n).
\end{equation*}
Conditional on $Y_t=y$ and $K=k$, the $k$ flipped positions form a uniformly random $k$-subset of $[n]$.
Let $I$ denote the number of flipped zero-bits (i.e.\ $0\to 1$ flips).
Then
\begin{equation*}
    I \mid (Y_t=y,\,K=k)\sim \Hyp(n,y,k),
\end{equation*}
the hypergeometric distribution counting the number of successes when drawing $k$ elements without replacement from a population of size $n$ containing $y$ successes.

Finally, let $A$ be the event that the offspring is accepted in the selection step, and write $\ind_A$ for its indicator.

\begin{proposition}
\label{prop:pre:explicit_drift}
For fixed $y\in\{0,1,\dots,n\}$ and random variables $K,I,\ind_A$ as above we have
\begin{equation*}
    \Delta(y)=\E{(2I-K)\ind_A\mid Y_t=y}.
\end{equation*}
\end{proposition}

\begin{proof}
If the mutation flips $k$ bits of which $i$ are zero-bits and the offspring is accepted, then the number of zero-bits decreases by $i$ and increases by $(k-i)$, hence
\begin{equation*}
    Y_{t+1}=y-i+(k-i)=y+k-2i,
\end{equation*}
so $Y_t-Y_{t+1}=2i-k$ on $A$ and $0$ otherwise.
Thus $Y_t-Y_{t+1}=(2I-K)\ind_A$.
Taking the conditional expectation given $Y_t=y$ yields the claim.
\end{proof}

\paragraph{Acceptance Probability.}
In each generation, parent and offspring are compared with respect to a linear fitness function with positive weights.
Consequently, the weights of all unchanged bits cancel in the comparison.
In particular, conditional on $(K=k,\,I=i)$, the acceptance probability depends only on $(k,i)$ and on the weight model, but not on~$y$.
We therefore define the acceptance probability
\begin{equation}
\label{eq:pre:acceptance_probability}
    p_A(k,i)\coloneqq \Prob{A\mid (K=k,\,I=i)}
\end{equation}
for $k>0$ and $0\le i\le k$. For convenience, we deviate from this definition for $k=i=0$ and set $p_A(0,0)\coloneqq 0$. This corresponds to the case the parent and offspring are identical, in which case the drift is trivially zero.

\paragraph{Drift at the Optimum.}
Since weights are positive, if $i=0$ then only $1\to 0$ flips occur and the fitness cannot increase, so $p_A(k,0)=0$ for all $k\in\N_0$ (including, by our convention, $k=0$).
In particular, if $Y_t=0$ then $I=0$, hence $\Delta(0)=0$.

Conditioning \Cref{prop:pre:explicit_drift} on $(K,I)$ yields the explicit representation
\begin{equation}
\label{eq:pre:conditioned_explicit_drift}
    \Delta(y)
    = \sum_{k=0}^{n}\Prob{K=k}\sum_{i=0}^{\Min{k}{y}}(2i-k)\,\Prob{I=i\mid (Y_t=y,\,K=k)}\,p_A(k,i).
\end{equation}

\subsection{Approximating the Drift}

For $y\in\{0,1,\dots,n\}$ we write
\begin{equation*}
    \alpha\coloneqq \frac{y}{n}\in[0,1]
\end{equation*}
for the fraction of zero-bits.
In many arguments we work with a function approximating the drift that is continuous in $\alpha$, which allows us to use standard calculus tools. For all $\chi>0$ and $\alpha\in[0,1]$ we define
\begin{equation}
\label{eq:pre:approximate_drift}
    D(\chi,\alpha)
    \coloneqq \sum_{k=0}^{\infty} e^{-\chi}\frac{\chi^k}{k!}\sum_{i=0}^{k} (2i-k)\binom{k}{i}\alpha^i(1-\alpha)^{k-i}\,p_A(k,i).
\end{equation}
For fixed $\chi>0$ we write $D_\chi(\alpha)\coloneqq D(\chi,\alpha)$.
Equivalently, if $K\sim\Poi(\chi)$ and $I\mid (K=k)\sim\Bin(k,\alpha)$, then
\begin{equation*}
    D(\chi,\alpha)=\E{(2I-K)\ind_A}.
\end{equation*}
Note that $D_\chi(0)=0$ as well, since $\alpha=0$ implies $I=0$ almost surely and $p_A(k,0)=0$ for all $k\in\N_0$.

\paragraph{Interpretation.}
The exact mutation size satisfies $K\sim\Bin(n,\chi/n)$, i.e.\ it is a sum of $n$ independent bit-flips.
For fixed $\chi$ and large $n$, $\Bin(n,\chi/n)$ is well-approximated by $\Poi(\chi)$ (e.g. \cite[Theorem~4.6]{rossFundamentalsSteinsMethod2011}).
Moreover, conditional on $K=k$, the exact number of flipped zero-bits is hypergeometric (sampling without replacement), which can be approximated by the binomial law $\Bin(k,\alpha)$ with $\alpha = y/n$ (sampling with replacement) when $k$ is of constant order and $n$ is large.

The next proposition makes the approximation error explicit.
It will be used to transfer results from the continuous approximation $D_\chi(\alpha)$ to the drift $\Delta(y)$.

\begin{proposition}
\label{prop:pre:drift_convergence}
Fix $\chi\in(0,\infty)$ and let $y(\alpha)\coloneqq\lfloor \alpha n\rfloor$.
There exists a constant $c(\chi)>0$, depending only on $\chi$, such that uniformly for all $\alpha\in[0,1]$,
\begin{equation*}
    \abs{\Delta(y(\alpha)) - D_\chi(\alpha)} \le \frac{c(\chi)}{n}.
\end{equation*}
In particular, for every fixed $\alpha\in[0,1]$ we have $\Delta(\lfloor \alpha n\rfloor) \to D_\chi(\alpha)$ as $n\to\infty$.
\end{proposition}

\begin{proof}
A detailed proof is given in \Cref{app:prop:pre:drift_convergence}.
\end{proof}

\subsection{Local Behaviour at the Optimum}

Close to the optimum we have $\alpha\approx 0$, i.e.\ $Y_t=\lfloor \alpha n\rfloor$ is small compared to~$n$.
In this regime, the sign of the drift is governed by the first-order term of $\alpha\mapsto D_\chi(\alpha)$ at $\alpha=0$.
Since $\alpha=0$ is a boundary point, we work with the right derivative $\partial_+D_\chi(0)$.

\begin{proposition}
\label{prop:pre:right_derivative}
Fix $\chi\in(0,\infty)$.
Independently of the weight model, the right-derivative of $D_\chi(\alpha)$ at $\alpha=0$ exists and satisfies
\begin{equation*}
    \partial_+ D_\chi(0)
    = \sum_{k = 1}^\infty e^{-\chi}\frac{\chi^k}{k!}\,k(2-k)\,p_A(k,1).
\end{equation*}
\end{proposition}

\begin{proof}
See \Cref{app:prop:pre:right_derivative}.
\end{proof}

\paragraph{Irwin--Hall distribution.}
For the Uniform weight model, acceptance probabilities can be expressed in terms of sums of i.i.d.\ uniform random variables.
Let $U_1,\dots,U_k$ be i.i.d.\ $\Unif(0,1)$ and set $S_k\coloneqq \sum_{j=1}^k U_j$.
The distribution of $S_k$ is called the Irwin--Hall distribution.
Its cumulative distribution function (cdf) on $[0,k]$ is given by
\begin{equation*}
    F_k(x)\coloneqq \Prob{S_k\le x}
    = \frac{1}{k!}\sum_{j=0}^{\lfloor x\rfloor}(-1)^j\binom{k}{j}(x-j)^k,
\end{equation*}
Moreover, $S_k$ is symmetric around $k/2$, i.e.\ $S_k\overset{d}{=}k-S_k$, since $1-U_j\sim\Unif(0,1)$ and
\begin{equation*}
    k-S_k=\sum_{j=1}^k (1-U_j)\ \overset{d}{=}\ \sum_{j=1}^k U_j = S_k.
\end{equation*}
Consequently, $F_k(x)=1-F_k(k-x)$ for all $x\in[0,k]$.

\subsection{Drift Theorems}

We use the following drift theorems (stated in our notation and in the direction matching our drift convention).
The multiplicative drift theorem \cite[Theorem~2.4.5]{lenglerDriftAnalysis2020} is used for the $\Oh{n\log n}$ regime, whereas the Simplified Drift Theorem \cite[Theorem~2]{oliveto2012erratumsimplifieddriftanalysis} is used to derive exponential lower bounds once we establish uniform negative drift on a constant interval.

\begin{theorem}[Multiplicative Drift, Upper Tail Bound \cite{lenglerDriftAnalysis2020}]
\label{thm:pre:mult_drift_upper_tail}
Let $(Y_t)_{t\ge 0}$ be a sequence of non-negative random variables with finite state space $S\subseteq \R_{\ge 0}$ such that $0\in S$.
Let $s_{\min}\coloneqq \min(S\setminus\{0\})$, and let $T \coloneqq \min\SetBuilder{t\ge 0}{Y_t = 0}$.
Suppose that $Y_0=s_0$ and that there exists $\delta>0$ such that, for all $s\in S\setminus\{0\}$ and all $t\ge 0$,
\begin{equation*}
    \E{Y_t - Y_{t+1}\mid Y_t=s}\ge \delta s.
\end{equation*}
Then for all $r\ge 0$,
\begin{equation*}
    \Prob*{T>\left\lceil \frac{r+\ln(s_0/s_{\min})}{\delta}\right\rceil}\le e^{-r}.
\end{equation*}
\end{theorem}

\smallskip

\begin{theorem}[Simplified Drift Theorem \cite{oliveto2012erratumsimplifieddriftanalysis}]
\label{thm:pre:simplified_drift}
Let $(Y_t)_{t\ge 0}$ be a Markov process over a finite state space $S\subseteq \R_{\ge 0}$.
Suppose there exist an interval $[a,b]$ in the state space, two constants $\delta,\eps>0$ and, possibly depending on $\ell\coloneqq b-a$, a function $r(\ell)$ satisfying $1\le r(\ell)=\oh{\ell/\log(\ell)}$ such that for all $t\ge 0$ the following two conditions hold:
\begin{enumerate}[leftmargin=*,label=(\arabic*)]
    \item $\E{Y_{t+1}-Y_t \mid Y_t=y}\ge \eps$ for all $a<y<b$,
    \item $\Prob{\abs{Y_{t+1}-Y_t} \ge j \mid Y_t=y}\le \frac{r(\ell)}{(1+\delta)^j}$ for $y > a$ and $j\in\N_0$.
\end{enumerate}
Then there is a constant $c^*>0$ such that for $T^*\coloneqq \min\SetBuilder{t\ge 0}{Y_t\le a}$ it holds $\Prob{T^* \le 2^{c^*\ell/r(\ell)}\mid Y_0\ge b} = 2^{-\Om{\ell/r(\ell)}}$. 
\end{theorem}

To apply \Cref{thm:pre:simplified_drift} it suffices to verify a uniform drift away from the target on an interval
and an exponentially decaying tail bound on the absolute step size $\abs{Y_{t+1}-Y_t}$.

\begin{lemma}
\label{lem:pre:simplified_drift_conditions}
Fix $\chi>0$.
Assume there exist constants $0<a<b<1$ and $\eps>0$, all independent of $n$, such that
\begin{equation*}
    D_\chi(\alpha)\le -2\eps
    \quad\text{for all }\alpha\in[a,b].
\end{equation*}
Then for all sufficiently large $n$ the following holds:
\begin{enumerate}[leftmargin=*,label=(\arabic*)]
    \item For all $\alpha\in(a,b)$ we have $\E{Y_{t+1}-Y_t \mid Y_t=\lfloor \alpha n\rfloor}\ge \eps$.
    \item For all $\alpha>a$ and all $j\in\N_0$ we have $\Prob{\abs{Y_{t+1}-Y_t}\ge j \mid Y_t=\lfloor \alpha n\rfloor}\le \frac{e^\chi}{2^j}$.
\end{enumerate}
\end{lemma}

\begin{proof}
(1) By \Cref{prop:pre:drift_convergence}, there exists $c(\chi)>0$ such that uniformly in $\alpha\in[0,1]$,
\begin{equation*}
    \abs{\Delta(\lfloor \alpha n\rfloor)-D_\chi(\alpha)}\le \frac{c(\chi)}{n}.
\end{equation*}
Choose $n_0$ such that $c(\chi)/n_0\le \eps$.
Then for all $n\ge n_0$ and all $\alpha\in[a,b]$,
\begin{equation*}
    \Delta(\lfloor \alpha n\rfloor)\le D_\chi(\alpha)+\frac{c(\chi)}{n}
    \le -2\eps+\eps=-\eps.
\end{equation*}
and hence $\E{Y_{t+1}-Y_t\mid Y_t=\lfloor \alpha n\rfloor}=-\Delta(\lfloor \alpha n\rfloor)\ge \eps$.

(2) Let $K_n\sim\Bin(n,\chi/n)$ be the mutation size.
In one generation the number of zero-bits changes by at most the number of flipped bits, hence
\begin{equation*}
    \abs{Y_{t+1}-Y_t}\le K_n.
\end{equation*}
Therefore for any state $y$ and any $j\in\N_0$,
\begin{equation*}
    \Prob{\abs{Y_{t+1}-Y_t}\ge j\mid Y_t=y}\le \Prob{K_n\ge j}.
\end{equation*}
For $j\ge 1$ we apply Markov's inequality to $e^{tK_n}$ with a parameter $t>0$:
\begin{equation*}
    \Prob{K_n\ge j}
    = \Prob{e^{tK_n}\ge e^{tj}}
    \le \frac{\E{e^{tK_n}}}{e^{tj}}.
\end{equation*}
Since $K_n\sim\Bin(n,\chi/n)$,
\begin{equation*}
    \E{e^{tK_n}}
    =\Bigl(1-\frac{\chi}{n}+\frac{\chi}{n}e^t\Bigr)^n
    \le e^{\chi(e^t-1)}.
\end{equation*}
Choosing $t=\ln 2$ gives $\E{e^{tK_n}}\le e^\chi$, hence for all $j\ge 1$,
\begin{equation*}
    \Prob{\abs{Y_{t+1}-Y_t}\ge j\mid Y_t=y}\le \frac{e^\chi}{2^j}.
\end{equation*}
For $j=0$ the bound holds as well since $e^\chi\ge 1$.
\end{proof}

\subsection{Drift below a Constant Threshold}

Since the approximation in \Cref{prop:pre:drift_convergence} comes with an additive error of order $\Oh{1/n}$, sign and drift statements proved for $D_\chi(\alpha)$ do not automatically transfer pointwise to the exact drift $\Delta(y)$.
Near the optimum we therefore typically obtain a clean multiplicative drift bound only down to a constant $y_0\le\Oh{1}$.
The remaining states $\{1,\dots,y_0\}$ are handled by the following lemma.
It uses no distribution-specific properties beyond positivity of the weights.

\begin{lemma}
\label{lem:pre:constant_drift}
Consider the $\opoEA$ in a dynamic environment where in each generation the fitness function is linear with positive weights. Let the mutation rate be $\chi/n$  for some constant $\chi>0$.
Suppose there exist constants $y_0\in\N$, $c>0$ and $n_0\in\N$, all independent of $n$, such that for all $n\ge n_0$ and all
$y\in\{y_0,y_0+1,\dots,n\}$ we have
\begin{equation*}
    \Delta(y)\ge \frac{c}{n}\,y.
\end{equation*}
Then for the runtime $T$ the following conditions hold:
\begin{enumerate}[leftmargin=*,label=(\arabic*)]
    \item For all $n$ sufficiently large we have $\sup_{1\le y\le y_0}\E{T\mid Y_0=y}\le\Oh{n}$.
    \item For every fixed $y\in\{1,\dots,y_0\}$ we have $\Prob{T>n\log n\mid Y_0=y}=\oh{1}$.
\end{enumerate}
\end{lemma}

\begin{proof}
A proof is given in \Cref{app:lem:pre:constant_drift}.
\end{proof}
\section{Permutation-Invariant Dynamic Environments}
\label{sec:addendum}

\paragraph{Permutation-Invariant Weight Vectors.} 
In this section we will define a class of functions that generalize dynamic linear functions and DBV into a joint framework, the class of permutation-invariant dynamic functions. Then we show that for all mutation parameters $\chi < \chi_{\mathrm{DBV}}$, the $\opoEA$ with mutation rate $\chi/n$ finds the optimum of every permutation-invariant dynamic function in time $\Oh{n \log n}$ in expectation and a.a.s.

While this result was claimed for dynamic linear functions in~\cite{lengler1+1EANoisyLinear2018}, and by extension for DBV in~\cite{lenglerMeier2024}, the proof there contained a severe gap, as we discuss in detail in Section~\ref{sec:erratum} below: the drift was analysed in the limit $\alpha \to 0^+$, and a coupling argument was given to show that the drift is positive everywhere if its limit for $\alpha\to 0^+$ is positive. However, this coupling was wrong, and we fix that with Theorem~\ref{thm:add:main} below, and at the same time generalize the result to all permutation-invariant dynamic functions. We do not see a way to rescue the coupling approach, so we employ a different strategy by finding an explicit lower bound for the drift for all $\alpha \in (0,1]$, and then show convexity of this lower bound. This is actually the main part of our proof, so the extension compared to the incomplete proof in~\cite{lengler1+1EANoisyLinear2018} is very substantial.

\begin{definition}\label{def:permutation-invariant}
A permutation-invariant dynamic function is given as follows. In each generation $t$, an independent positive random weight vector $W^{(t)}=(W_1^{(t)},\dots,W_n^{(t)})\in \R^n_{> 0}$ is sampled from a permutation-invariant distribution, i.e.\
\begin{equation*}
    (W_1^{(t)},\dots,W_n^{(t)}) \overset{d}{=}(W_{\pi(1)}^{(t)},\dots,W_{\pi(n)}^{(t)})\quad\text{for every permutation }\pi\in S_n.
\end{equation*}
\end{definition}

\begin{remark}
    The class of permutation-invariant dynamic functions includes the dynamic linear function model from \cite{lengler1+1EANoisyLinear2018}, where $W_j^{(t)}\overset{\text{i.i.d.}}{\sim}\mathcal{D}$, and weight models such as DBV, where $W^{(t)}$ is a uniform random permutation of $(1,2,4,\dots,2^{n - 1})$.
    
    Note that we do \emph{not} require that the $n$ weights $W_1^{(t)},\dots,W_n^{(t)}$ are independent of each other. In particular, for the DBV benchmark they are not, since each weight $2^0,\ldots,2^{n-1}$ appears exactly once. 
\end{remark}

Recall that the $\opoEA$ on permutation-invariant dynamic functions bases its selection in generation $t$ on the same weight vector $W^{(t)}$ for parent and offspring, as described in \Cref{alg:pre:opoea}.

The main result of this section is the following.

\begin{theorem}
\label{thm:add:main}
    Let $\chi_0\approx 1.59362$ be the unique strictly positive root of $2-\chi-2e^{-\chi}$. For every constant $0<\chi<\chi_0$ and any permutation-invariant dynamic function, the runtime of the $\opoEA$ with mutation parameter $\chi$ is $\Oh{n\log n}$ in expectation and a.a.s.
\end{theorem}

The remainder of the section will be devoted to proving \Cref{thm:add:main}. We start with an elementary observation.

\begin{proposition}
    For a permutation-invariant dynamic function, conditional on $(K = k, I = i)$, the acceptance probability $p_A(k,i)$ in \Cref{eq:pre:acceptance_probability} depends only on $k$ and $i$.
\end{proposition}

\begin{proof}
Fix a generation $t$ and write $W\coloneqq W^{(t)}$. Conditional on $K=k$, the set $F\subseteq[n]$ of flipped positions is a uniformly random $k$-subset of $[n]$.
Conditional on $(K=k,I=i)$, the set $G\subseteq F$ of flipped zero-bits is a uniformly random $i$-subset of $F$.
The acceptance event $A$ is determined by
\begin{equation*}
    \sum_{j\in G} W_j \ge \sum_{j\in F\setminus G} W_j.
\end{equation*}
By permutation-invariance of $W$, the distribution of the weights that end up with flipped bits depends only on $\abs{F}=k$, and conditional on $(K=k,I=i)$ the partition $(G,F\setminus G)$ is uniform among all partitions of $F$ into parts of size $i$ and $k-i$.
Consequently, the conditional probability is a function of $(k,i)$ only.
\end{proof}

Next we provide a useful bound on the acceptance probability.

\begin{lemma}
\label{lem:add:pa_clean_bounds}
For fixed $k\ge 1$ and $0\le i\le k$ we have
\begin{equation}
\label{eq:add:pa_bounds_interval}
\Max*{0}{\frac{2i}{k}-1} \le  p_A(k,i) \le \Min*{1}{\frac{2i}{k}}.
\end{equation}
Moreover, for all $k\ge 3$,
\begin{equation}
\label{eq:add:pa_one_flip}
p_A(k,1)\le \frac1k.
\end{equation}
\end{lemma}

\begin{proof}
Condition on $(K=k,I=i)$ in some fixed generation and let $F$ be the set of flipped positions and
$G\subseteq F$ the set of improving flips.
Set
\begin{equation*}
S\coloneqq\sum_{j\in G} W^{(t)}_j,\qquad
T\coloneqq\sum_{j\in F\setminus G} W^{(t)}_j.
\end{equation*}
Since weights are strictly positive we have $S+T>0$ and acceptance is the event $A=\{S\ge T\}$.
Define $R\coloneqq S/(S+T)\in[0,1]$, so $A=\{R\ge1/2\}$.

By permutation-invariance and the fact that conditional on $(K=k,I=i)$ the set $G$ is a uniform $i$-subset of $F$,
\begin{equation*}
\E{R\mid K=k,I=i}=\frac{i}{k}.
\end{equation*}
For the upper bound in \eqref{eq:add:pa_bounds_interval}, apply Markov's inequality to the nonnegative random variable $R$:
\begin{equation*}
p_A(k,i)=\Prob{R\ge 1/2}\le \frac{\E{R}}{1/2}=\frac{2i}{k}.
\end{equation*}
For the lower bound, note that $R\le 1$ on $A$ and $R\le 1/2$ on $\lnot A$, hence
\begin{equation*}
\E{R}\le 1\cdot \Prob{A} + \frac12\cdot\Prob{\lnot A}\le\frac{1}{2}+\frac{1}{2}\,\Prob{A},
\end{equation*}
which rearranges to $\Prob{A}\ge 2\E{R}-1 = 2i/k-1$.
Combining these bounds with $0\le p_A(k,i)\le 1$ yields \eqref{eq:add:pa_bounds_interval}.

Finally, let $i = 1$ and assume $k \ge 3$. If the offspring is accepted, then the single improving weight is at least the sum of the remaining $k - 1$ weights, that is,
\begin{equation*}
    W_j \ge \sum_{\ell \in F \setminus \{j\}} W_\ell.
\end{equation*}
Moreover, at most one index in $F$ can satisfy this inequality. Indeed, if distinct $j,j'\in F$ both satisfied it, then since $k\ge 3$ and all weights are positive, we would have
\begin{equation*}
    W_j\ge W_{j'} + \sum_{\ell \in F\setminus\{j,j'\}} W_\ell > W_{j'}\qquad\text{and}\qquad W_{j'}\ge W_{j} + \sum_{\ell \in F\setminus\{j,j'\}} W_\ell > W_{j},
\end{equation*}
a contradiction. Hence there is at most one position in $F$ whose weight is at least the sum of all other weights.
By permutation-invariance, each position is distinguished with probability at most $1/k$ which proves \Cref{eq:add:pa_one_flip}.
\end{proof}

\paragraph{A Lower Bound for $D_\chi(\alpha)$.}
We now come to the heart of the proof, a lower bound for the drift that holds for all $\alpha$ and that is convex. This is the point where we substantially deviate from the previous approach in~\cite{lengler1+1EANoisyLinear2018}.

For integers $k\ge 0$ and $0\le i\le k$, define 
\begin{equation*}
\widetilde{p}_A(k,i)\coloneqq
\begin{cases}
0, & i=0,\\
1/k, & i = 1\text{ and } k\ge 3,\\
2i/k, & 1 < i<k/2,\\
0, & i=k/2,\\
2i/k-1, & k/2<i\le k.
\end{cases}
\end{equation*}
Moreover, for convenience we set $\widetilde{p}_A(k,i) = 0$ for $k<i$, which corresponds to the impossible case $K < I$. 
Replace $p_A(k,i)$ in \Cref{eq:pre:approximate_drift} with $\widetilde{p}_A(k,i)$ to get
\begin{equation*}
\widetilde{D}_\chi(\alpha)
\coloneqq\sum_{k = 0}^\infty e^{-\chi}\frac{\chi^k}{k!}\sum_{i=0}^k (2i-k)\binom{k}{i}\alpha^i(1-\alpha)^{k-i}\,\widetilde{p}_A(k,i).
\end{equation*}

We first check that the formula indeed provides a lower bound for the drift.

\begin{lemma}
\label{lem:add:tilde_approximative_drift}
For every permutation-invariant dynamic function,
\begin{equation*}
D_\chi(\alpha)\ \ge\ \widetilde{D}_\chi(\alpha)\qquad\text{for all }\alpha\in[0,1].
\end{equation*}
\end{lemma}

\begin{proof}
Fix $k,i$.
If $2i-k<0$ then \Cref{lem:add:pa_clean_bounds} gives $p_A(k,i)\le \widetilde{p}_A(k,i)$, hence
$(2i-k)p_A(k,i)\ge (2i-k)\widetilde{p}_A(k,i)$.
If $2i-k>0$ then \Cref{lem:add:pa_clean_bounds} gives $p_A(k,i)\ge \widetilde{p}_A(k,i)$, hence the same inequality holds.
In the remaining case $2i-k=0$, the corresponding term in the drift expansion vanishes, hence the inequality holds trivially.
Multiplying the inequality $(2i-k)p_A(k,i)\ge (2i-k)\widetilde{p}_A(k,i)$ by the nonnegative factor
$e^{-\chi}\chi^k/k!\,\binom{k}{i}\alpha^i(1-\alpha)^{k-i}$ and summing over all $i\in\{0,\dots,k\}$ and $k\ge 0$
yields $D_\chi(\alpha)\ge \widetilde{D}_\chi(\alpha)$.
\end{proof}

Crucially, as the next lemma shows, the lower bound $\widetilde{D}_\chi(\alpha)$ is sharp in the limit $\alpha \to 0^+$, i.e., there it's derivative converges to the exact value of $\partial_+{D}_\chi(0) = 2-\chi-2e^{-\chi}$, as it has been computed in~\cite{lengler1+1EANoisyLinear2018}. 

\begin{lemma}
\label{lem:add:right_derivative_tildeD}
The right derivative of $\widetilde{D}_\chi(\alpha)$ at $\alpha = 0$ satisfies
\begin{equation*}
\partial_+\widetilde{D}_\chi(0)= 2-\chi-2e^{-\chi}.
\end{equation*}
Let $\chi_0\approx 1.59362$ denote the unique positive root of $2-\chi-2e^{-\chi}$.
Then $2-\chi-2e^{-\chi}>0$ for $\chi\in(0,\chi_0)$ and $2-\chi-2e^{-\chi}<0$ for $\chi\in(\chi_0,\infty)$.
\end{lemma}

\begin{proof}
Analogously to~\Cref{prop:pre:right_derivative} and its proof in \Cref{app:prop:pre:right_derivative}, 
we may compute the right derivative of $\widetilde{D}_\chi$ at $0$ as
\begin{equation*}
\partial_+\widetilde{D}_\chi(0)
=\sum_{k = 1}^\infty e^{-\chi}\frac{\chi^k}{k!}\,k(2-k)\,\widetilde{p}_A(k,1).
\end{equation*}
By definition, $\widetilde{p}_A(1,1)=1$, and $\widetilde{p}_A(k,1)=1/k$ for all $k\ge 3$, while the term $\widetilde{p}_A(2,1)$ for $k=2$ is multiplied with $2-k=0$.
Hence
\begin{align*}
\partial_+\widetilde{D}_\chi(0)
&= \sum_{k = 1}^\infty e^{-\chi}\frac{\chi^k}{k!}\,(2-k) \\
&= \sum_{k = 0}^\infty e^{-\chi}\frac{\chi^k}{k!}\,(2-k)-2e^{-\chi}.
\end{align*}
Since $\sum_{k = 0}^\infty p_k(\chi)=1$ and $\sum_{k=0}^\infty k\,p_k(\chi)=\chi$ for $p_k(\chi)=e^{-\chi}\chi^k/k!$,
the full sum equals $2-\chi$, and thus $\partial_+\widetilde{D}_\chi(0)=2-\chi-2e^{-\chi}$.

For the root statement, consider $g(\chi)=2-\chi-2e^{-\chi}$.
Then $g(0)=0$, $g'(\chi)=-1+2e^{-\chi}$ with $g'(0)=1$, and $g''(\chi)=-2e^{-\chi}<0$,
so $g$ is strictly concave and increases initially, but satisfies $g(\chi)\to-\infty$ as $\chi\to\infty$.
Hence $g$ has exactly one positive root $\chi_0$, and the sign change is as claimed.
\end{proof}

\paragraph{Decomposition by Mutation Size.}
Write $p_k(\chi)\coloneqq e^{-\chi}\chi^k/k!$ for the Poisson probability mass function (pmf) and
$b_{k,i}(\alpha)\coloneqq \binom{k}{i}\alpha^i(1-\alpha)^{k-i}$ for the binomial pmf.
For each $k\in\N_0$ define
\begin{equation*}
\widetilde{H}_k(\alpha)\coloneqq \sum_{i=0}^{k}\widetilde{h}_k(i)\,b_{k,i}(\alpha),
\quad\text{where}\quad
\widetilde{h}_k(i)\coloneqq (2i-k)\,\widetilde{p}_A(k,i).
\end{equation*}
Then, for all $\chi>0$ and $\alpha\in[0,1]$,
\begin{equation*}
\widetilde{D}_\chi(\alpha)=\sum_{k=0}^{\infty} p_k(\chi)\,\widetilde{H}_k(\alpha).
\end{equation*}

We now give a bound for the tail of this decomposition in which at least 5 bits are flipped.

\begin{lemma}
\label{lem:add:convexity_tail}
For all $\chi\in(0,\chi_0)$ there exists a constant $c>0$ such that for all $\alpha\in[0,1]$,
\begin{equation*}
\abs*{\sum_{k = 5}^\infty p_k(\chi)\,\widetilde{H}''_k(\alpha)}
\le c\,p_{5}(\chi).
\end{equation*}
\end{lemma}

\begin{proof}
Fix $\chi\in(0,\chi_0)$.
For all $k\ge 5$ we have
\begin{equation*}
\frac{p_{k+1}(\chi)}{p_k(\chi)}=\frac{\chi}{k+1}\le \frac{\chi}{6}.
\end{equation*}
Telescoping yields
\begin{equation}\label{eq:add:pois_tail_geom}
p_k(\chi)\le p_{5}(\chi)\Bigl(\frac{\chi}{6}\Bigr)^{k-5}\quad\text{for all }k\ge 5.
\end{equation}

For $k\ge 2$, differentiating $b_{k,i}$ twice yields the standard second-difference identity
\begin{equation}\label{eq:add:second_deriv_diff}
\widetilde{H}''_k(\alpha)
= k(k-1)\sum_{i=0}^{k-2}\bigl(\widetilde{h}_k(i+2)-2\widetilde{h}_k(i+1)+\widetilde{h}_k(i)\bigr)\,b_{k-2,i}(\alpha).
\end{equation}
Moreover, we claim that for all $k\ge 5$ and $0\le i \le k - 2$ we have
\begin{equation}\label{eq:add:second_diff_bound}
\abs{\widetilde{h}_k(i+2)-2\widetilde{h}_k(i+1)+\widetilde{h}_k(i)}\le 2.
\end{equation}
For $k =5,6$, \Cref{eq:add:second_diff_bound} is checked directly. For $k\ge 7$, by the definition of $\widetilde{p}_A(k,i)$, the function $i\mapsto \widetilde{h}_k(i)$ is piecewise quadratic, with breakpoints only at $i\in\{0,1,\lfloor k/2\rfloor,\lceil k/2\rceil\}$. Hence every window $(i,i+1,i+2)$ that does not meet a breakpoint lies entirely in one quadratic branch, and one checks that in this case we have
\begin{equation*}
    \abs{\widetilde h_k(i+2)-2\widetilde h_k(i+1)+\widetilde h_k(i)}=\frac{8}{k}\le 2.
\end{equation*}
For the finitely many windows $(i,i+1,i+2)$ that cross a breakpoint, a case distinction on the breakpoints and direct computations using the definition of $\widetilde{h}_k$ show that the relevant second differences have absolute value at most $2$.

Inserting \eqref{eq:add:second_diff_bound} into \eqref{eq:add:second_deriv_diff} and using $\sum_{i=0}^{k-2} b_{k-2,i}(\alpha)=1$ gives
\begin{equation}\label{eq:add:Hpp_bound_k}
\abs{\widetilde{H}''_k(\alpha)}\le 2k(k-1)\quad\text{for all }k\ge 2.
\end{equation}

Combining \eqref{eq:add:pois_tail_geom} and \eqref{eq:add:Hpp_bound_k} yields
\begin{align*}
\abs*{\sum_{k=5}^{\infty} p_k(\chi)\,\widetilde{H}''_k(\alpha)}
&\le 2\,p_5(\chi)\sum_{k=5}^{\infty} k(k-1)\Bigl(\frac{\chi}{6}\Bigr)^{k-5} \\
&= 2\,p_5(\chi)\sum_{j=0}^{\infty} (j+5)(j+4)\Bigl(\frac{\chi}{6}\Bigr)^{j}.
\end{align*}
The series has nonnegative coefficients and is therefore increasing in $\chi$ for $\chi\in(0,6)$.
Since $\chi_0<8/5<6$, it is bounded by its value at $\chi=8/5$, i.e.\ with $r\coloneqq \frac{8}{30}=\frac{4}{15}$.
Thus we may take
\begin{equation*}
c \coloneqq 2\sum_{j=0}^{\infty} (j+5)(j+4)\,r^{j}
= \frac{86760}{1331}\approx 65.184.
\end{equation*}
\end{proof}

With the tail bound, we are now able to prove convexity of $\widetilde{D}_\chi(\alpha)$. Convexity will then imply that positive value and slope for $\alpha=0$ imply positivity for all $\alpha \in [0,1]$. This is the step that replaces the coupling approach in~\cite{lengler1+1EANoisyLinear2018}.

\begin{lemma}
\label{lem:add:tildeD_convex}
For each fixed $\chi\in(0,\chi_0)$, the function $\alpha\mapsto \widetilde{D}_\chi(\alpha)$ is convex on $[0,1]$.
\end{lemma}

\begin{proof}
We have $\widetilde{D}_\chi(\alpha)=\sum_{k=0}^\infty p_k(\chi)\widetilde{H}_k(\alpha)$, hence
\begin{equation*}
\widetilde{D}_\chi''(\alpha)=\sum_{k=0}^\infty p_k(\chi)\,\widetilde{H}_k''(\alpha).
\end{equation*}
For $k=0,1$ we have $\widetilde{H}_k''(\alpha)=0$. For $k=2,3,4$, a direct computation from the definition of $\widetilde{p}_A$ yields
\begin{equation*}
\widetilde{H}_2(\alpha)=2\alpha^2,\qquad
\widetilde{H}_3(\alpha)=\alpha^3+3\alpha^2-\alpha,\qquad
\widetilde{H}_4(\alpha)=2\alpha^4-2\alpha^3+6\alpha^2-2\alpha,
\end{equation*}
and therefore
\begin{equation*}
\widetilde{H}_2''(\alpha)=4,\qquad
\widetilde{H}_3''(\alpha)=6\alpha+6\ge 6,\qquad
\widetilde{H}_4''(\alpha)=24\left(\alpha - \frac{1}{4}\right)^2 + \frac{21}{2} \ge \frac{21}{2}.
\end{equation*}
Using \Cref{lem:add:convexity_tail} (with tail starting at $k=5$) we obtain, for all $\alpha\in[0,1]$,
\begin{align*}
\widetilde{D}_\chi''(\alpha)
&= \sum_{k=2}^{4} p_k(\chi)\widetilde{H}_k''(\alpha) \;+\; \sum_{k=5}^{\infty} p_k(\chi)\widetilde{H}_k''(\alpha) \\
&\ge 4p_2(\chi)+6p_3(\chi)+\frac{21}{2}p_4(\chi) \;-\; c\,p_5(\chi).
\end{align*}

Since $p_k(\chi)=e^{-\chi}\chi^k/k!$, the right-hand side equals $e^{-\chi}Q(\chi)$ with
\begin{equation*}
Q(\chi)=2\chi^2+\chi^3+\frac{21}{48}\chi^4-\frac{c}{120}\chi^5
=2\chi^2+\chi^3+\frac{7}{16}\chi^4-\frac{723}{1331}\chi^5,
\end{equation*}
where we used $c=\frac{86760}{1331}$ from \Cref{lem:add:convexity_tail}.
We claim that $Q(\chi)>0$ for all $\chi\in(0,8/5]$, and hence in particular for all $\chi\in(0,\chi_0)$.

Indeed, write 
\begin{equation*}
Q(\chi) = \chi^2\bigl(2 + \chi R(\chi)\bigr),
\qquad R(\chi)\coloneqq 1 + \frac{7}{16}\chi - \frac{723}{1331}\chi^2.
\end{equation*}
The function $R$ is a concave quadratic, so its minimum on $[0,8/5]$ is attained at one of the endpoints.
Moreover,
\begin{equation*}
R(0) = 1\qquad\text{and}\qquad R(8/5)=\frac{20591}{66550}>0.
\end{equation*}
Hence $R(\chi)>0$ for all $\chi\in[0,8/5]$, and therefore
\begin{equation*}
Q(\chi) = \chi^2(2+\chi R(\chi))>0\qquad\text{for all }\chi\in(0,8/5).
\end{equation*}

Consequently, $\widetilde{D}_\chi''(\alpha)\ge e^{-\chi}Q(\chi)>0$ for all $\alpha\in[0,1]$ and $\chi\in(0,\chi_0)$, which proves convexity.
\end{proof}

We are now ready to put everything together. As announced earlier, convexity implies that the drift is positive for all $\alpha \in (0,1]$.

\begin{corollary}
\label{cor:add:global_positivity_tilde}
Let $\chi\in(0,\chi_0)$ and set $g(\chi)\coloneqq 2-\chi-2e^{-\chi}>0$.
Then for all $\alpha\in[0,1]$,
\begin{equation*}
\widetilde{D}_\chi(\alpha)\ge \alpha\,\partial_+\widetilde{D}_\chi(0)
=\alpha\,g(\chi),
\end{equation*}
and in particular $\widetilde{D}_\chi(\alpha)> 0$ for all $\alpha\in(0,1]$.
\end{corollary}

\begin{proof}
By \Cref{lem:add:tildeD_convex}, the function $\widetilde{D}_\chi$ is convex on $[0,1]$ and satisfies $\widetilde{D}_\chi(0)=0$.
Hence $\widetilde{D}_\chi(\alpha)\ge \alpha\,\partial_+\widetilde{D}_\chi(0)$ for all $\alpha\in[0,1]$.
The derivative is given by \Cref{lem:add:right_derivative_tildeD}, and $g(\chi)>0$ for $\chi\in(0,\chi_0)$ by definition of $\chi_0$.
\end{proof}

We are now ready to prove the main theorem of this section.

\subsection{Proof of \Cref{thm:add:main}}

\begin{proof}
Fix $\chi\in(0,\chi_0)$ and write $g(\chi)\coloneqq 2-\chi-2e^{-\chi}>0$.
By combining \Cref{lem:add:tilde_approximative_drift} and \Cref{cor:add:global_positivity_tilde}, we have for all $\alpha\in[0,1]$,
\begin{equation}\label{eq:add:global_positivity}
D_\chi(\alpha)\ge g(\chi)\alpha.
\end{equation}
By \Cref{prop:pre:drift_convergence}, there exists a constant $c(\chi)>0$ such that, uniformly in $\alpha\in[0,1]$,
\begin{equation}\label{eq:add:drift_convergence}
\abs*{\Delta(\lfloor \alpha n\rfloor)-D_\chi(\alpha)}\le \frac{c(\chi)}{n}.
\end{equation}
Let
\begin{equation*}
y_0\coloneqq \left\lceil \frac{2c(\chi)}{g(\chi)}\right\rceil.
\end{equation*}
Then for all $y\ge y_0$, setting $\alpha=y/n$ and using \eqref{eq:add:global_positivity} and \eqref{eq:add:drift_convergence} yields
\begin{equation}\label{eq:add:mult_drift_lower}
\Delta(y)\ge g(\chi)\frac{y}{n}-\frac{c(\chi)}{n}\ge\frac{g(\chi)}{2n}\,y.
\end{equation}

Define the stopping time $T_{y_0}\coloneqq \min\{t\ge 0 : Y_t\le y_0\}$ and the shifted process
\begin{equation*}
Z_t \coloneqq \max\{Y_t-y_0,0\}.
\end{equation*}
Then $Z_t\in\{0,1,\dots,n-y_0\}$ and $T_{y_0}=\min\{t\ge 0: Z_t=0\}$.
Moreover, for all $s\ge 1$,
\begin{equation*}
\E{Z_t-Z_{t+1}\mid Z_t=s}
= \Delta(s+y_0)
\ge \frac{g(\chi)}{2n}(s+y_0)
\ge \frac{g(\chi)}{2n}\,s,
\end{equation*}
where we used \eqref{eq:add:mult_drift_lower}.
Hence \Cref{thm:pre:mult_drift_upper_tail} applies to $(Z_t)_{t\ge 0}$ and yields $T_{y_0}=\Oh{n\log n}$ a.a.s.
In particular, the exponential tail bound in \Cref{thm:pre:mult_drift_upper_tail} implies
\begin{equation}\label{eq:add:runtime_y0}
\E{T_{y_0}} = \Oh{n\log n}.
\end{equation}

To cover the remaining states $\{1,\dots,y_0\}$ we invoke \Cref{lem:pre:constant_drift}.
Its assumptions are satisfied since weights are strictly positive almost surely and the multiplicative drift lower bound
\eqref{eq:add:mult_drift_lower} holds for all $y\ge y_0$.
Consequently, for all $n$ sufficiently large,
\begin{equation}\label{eq:add:endgame_exp}
\sup_{1\le y\le y_0}\E{T\mid Y_0=y}=\Oh{n},
\end{equation}
and
\begin{equation}\label{eq:add:endgame_tail}
\sup_{1\le y\le y_0}\Prob{T>n\log n\mid Y_0=y}=\oh{1}.
\end{equation}

We now combine the two parts. Recall $T=\min\{t\ge 0:Y_t=0\}$ and $T_{y_0}=\min\{t\ge 0:Y_t\le y_0\}$.
At time $T_{y_0}$ we have $Y_{T_{y_0}}\in\{0,1,\dots,y_0\}$.
Since $(Y_t)_{t\ge 0}$ is a time-homogeneous Markov chain, conditioning on $Y_{T_{y_0}}$ yields
\begin{align*}
\E{T}
= \E{T_{y_0}} + \E{\E{T-T_{y_0}\mid Y_{T_{y_0}}}}
&\le \E{T_{y_0}} + \sup_{1\le y\le y_0}\E{T\mid Y_0=y} \\
&= \Oh{n\log n},
\end{align*}
where we used \eqref{eq:add:runtime_y0} and \eqref{eq:add:endgame_exp}.

Moreover, conditioning on $Y_{T_{y_0}}$ and using \eqref{eq:add:endgame_tail} gives
\begin{equation*}
\Prob{T-T_{y_0}>n\log n}
\le \sup_{1\le y\le y_0}\Prob{T>n\log n\mid Y_0=y}
= \oh{1}.
\end{equation*}
Together with $T_{y_0}=\Oh{n\log n}$ a.a.s., a union bound implies $T=\Oh{n\log n}$ a.a.s.
\end{proof}
\section{Dynamic Binary Value}
\label{sec:dbv}

In this chapter we analyse the $\opoEA$ on the dynamic binary value (DBV) problem \cite{lenglerMeier2024}.
In each generation, the weight vector is a uniformly random permutation of $(1,2,4,\dots,2^{n-1})$ so all weights are distinct and positive. As in \Cref{sec:preliminaries}, we study the Markov chain $(Y_t)_{t\ge 0}$ where
$Y_t$ denotes the number of zero-bits of the current search point.

\paragraph{Motivation and Results.}
Vermetten~et~al.\ \cite{vermettenEmpiricalAnalysisDynamic2024} report that for larger mutation rates the $\opoEA$ fails to get close to the optimum on DBV even with a generous evaluation budget of $1000n$.
A structural explanation was given in~\cite{lengler1+1EANoisyLinear2018,lenglerMeier2024}, where it was shown that above the critical mutation parameter $\chi_\mathrm{dbv}$, the drift is negative close to the optimum and thus the algorithm does not come close to the optimum. However, those papers did not make any statement on where exactly the algorithm would stall.
We give a substantially more refined picture. For $\chi > \chi_\mathrm{dbv}$, define
\begin{equation}
\label{eq:dbv:alpha_star_cont}
    \alpha^*(\chi)
    \coloneqq \frac{2-\chi-2e^{-\chi}}{2-2\chi-2e^{-\chi}}.
\end{equation}
Then we show that the $\opoEA$ with mutation rate $\chi/n$ reaches distance $\alpha^*(\chi)$ quickly, but then progress comes to a complete halt for an exponential time.


\begin{theorem}
\label{thm:dbv:main}
Let $\chi_\mathrm{dbv}\approx 1.59362$ be the unique strictly positive root of $2-\chi-2e^{-\chi}$.
Consider the $\opoEA$ with mutation parameter $\chi$ on DBV.
\begin{enumerate}[leftmargin=*,label=(\arabic*)]
    \item For every fixed $0<\chi<\chi_\mathrm{dbv}$, the runtime is $\Oh{n\log n}$ in expectation and a.a.s.
    \item For every fixed $\chi>\chi_\mathrm{dbv}$, the runtime is $2^{\Om{n}}$ in expectation and a.a.s. \item For fixed $\chi>\chi_\mathrm{dbv}$, consider $\alpha^*(\chi)$ from \Cref{eq:dbv:alpha_star_cont}. For every fixed $\alpha\in[0,\alpha^*(\chi))$ the hitting time of $\{0,1,\dots,\lfloor \alpha n\rfloor\}$ is $2^{\Om{n}}$ in expectation and a.a.s., and for every fixed $\alpha\in(\alpha^*(\chi),1]$ the hitting time of $\{0,1,\dots,\lfloor \alpha n\rfloor\}$ is $\Oh{n}$ in expectation and a.a.s.
    \end{enumerate}
\end{theorem}

\begin{figure}[h]
    \centering
    \includegraphics[width=1.0\linewidth]{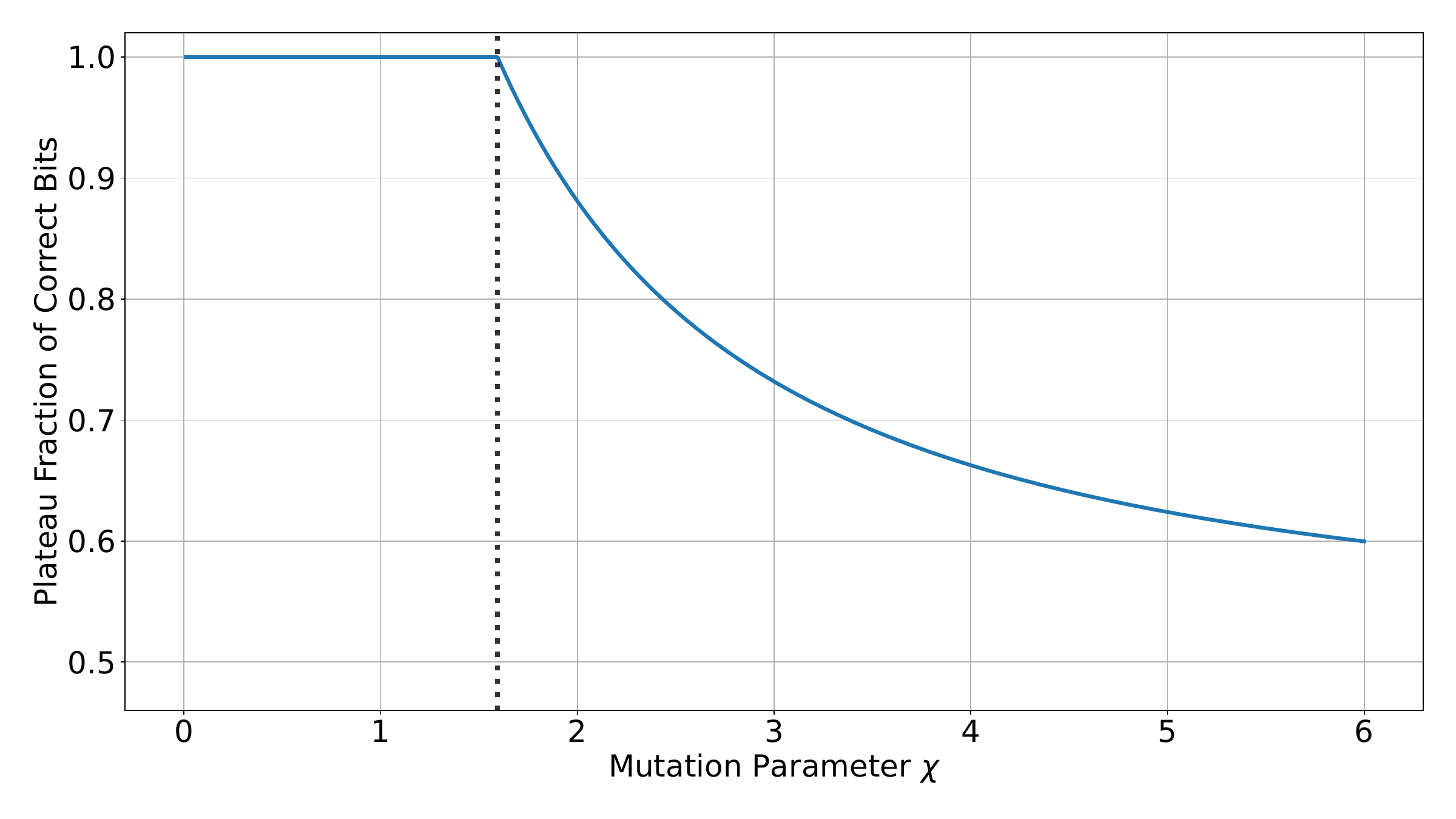}
    \caption{Visualization of the DBV plateau location from \Cref{thm:dbv:main}~(3). For $0<\chi< \chi_{\mathrm{dbv}}$, the $\opoEA$ reaches the optimum efficiently, so the plotted plateau fraction of correct bits is $1$. For $\chi > \chi_{\mathrm{dbv}}$, the algorithm gets stuck at plateau fraction $1-\alpha^*(\chi)$, where $\alpha^*(\chi)$ is given in \Cref{eq:dbv:alpha_star_cont}. The dotted line marks the critical threshold $\chi_{\mathrm{dbv}} \approx 1.59362$.}
    \label{fig:dbv:plateau}
\end{figure}

As discussed earlier, the proof in~\cite{lenglerMeier2024,lengler1+1EANoisyLinear2018} has a gap, but the gap only affected part (1) of the statement, not part (2). However, the statement there is slightly weaker than our result in (2), because it only claimed a superpolynomial runtime, whereas we show an exponential one. Therefore, we also provide a proof of (2), even though it is just a small extension compared to previous work. Moreover, in order for the section to be self-contained we also provide a proof of (1) even though it is also implied by the more general statement of Theorem~\ref{thm:add:main} in Section~\ref{sec:addendum}. As mentioned before, there is no analogue in the literature to part (3) of the statement. 

In Section~\ref{sec:dbv:finite_n} we will give explicit formulas for the drift and for the location of the plateau even for finite values of $n$.

\subsection{Acceptance Probability}
\label{sec:dbv:acceptance}

Recall from \Cref{eq:pre:acceptance_probability} that $p_A(k,i)$ denotes the probability of accepting the offspring conditional on
flipping exactly $k$ bits in total, among which exactly $i$ are zero-bits.

For DBV, the acceptance decision is determined by the flipped position with the largest weight.
Since all weights are distinct powers of two, the total fitness change is dominated by the largest-weight flip:
the offspring is accepted if and only if the largest-weight flipped bit changes from $0$ to $1$.
Under the random permutation model, conditional on flipping a fixed set of $k$ positions, the largest weight among these
positions is equally likely to be attached to any of the $k$ flipped positions.

\begin{proposition}
\label{prop:dbv:acceptance}
For DBV, for all integers $k\ge 1$ and $0\le i\le k$,
\begin{equation*}
    p_A(k,i)=\frac{i}{k}.
\end{equation*}
\end{proposition}

\begin{proof}
Fix $k\ge 1$ and condition on the event that mutation flips exactly $k$ positions, among which exactly $i$ positions are currently zero-bits.
Let $S$ be this set of flipped positions.
In DBV, the weight vector in the current generation is a uniformly random permutation of distinct weights, hence the relative order of weights on $S$ is uniform over all $k!$ permutations.
In particular, the position in $S$ carrying the maximum weight is uniformly random over the $k$ elements of $S$.
The offspring is accepted if and only if this maximum-weight flipped position is a $0\to 1$ flip, which holds with probability $i/k$.
\end{proof}

\subsection{Continuous Drift, Threshold, and Plateau}
\label{sec:dbv:continuous}

We now compute the continuous drift approximation $D(\chi,\alpha)$ from \Cref{eq:pre:approximate_drift} in closed form for DBV.

\begin{proposition}
\label{prop:dbv:drift_closed_form}
For DBV, for all $(\chi,\alpha)\in(0,\infty)\times[0,1]$,
\begin{equation}
\label{eq:dbv:drift_closed_form}
    D(\chi,\alpha)
    = \alpha\Bigl(\bigl(2-\chi-2e^{-\chi}\bigr)-\alpha\bigl(2-2\chi-2e^{-\chi}\bigr)\Bigr).
\end{equation}
\end{proposition}

\begin{proof}
The proof is a direct computation and is deferred to \Cref{sec:app:dbv:drift_closed_form}.
\end{proof}

\paragraph{Roots and Sign.}
The expression \Cref{eq:dbv:drift_closed_form} is a quadratic polynomial in $\alpha$ and is already factorised.
In particular, $\alpha=0$ is always a root.
Solving the bracket in \Cref{eq:dbv:drift_closed_form} for~$\alpha$ yields the second root
\begin{equation*}
    \alpha^*(\chi)
    = \frac{2-\chi-2e^{-\chi}}{2-2\chi-2e^{-\chi}},
\end{equation*}
which matches the definition of $\alpha^*(\chi)$ in \eqref{eq:dbv:alpha_star_cont}. Since $2-2\chi-2e^{-\chi}<0$ for every $\chi>0$, the bracket is strictly increasing in~$\alpha$.
It is therefore enough to understand the sign of
\begin{equation*}
    A(\chi)\coloneqq 2-\chi-2e^{-\chi},
\end{equation*}
which is exactly the value of the bracket at $\alpha=0$.
Let $\chi_\mathrm{dbv}\approx 1.59362$ denote the unique strictly positive root of $A(\chi)$.

\begin{corollary}
\label{cor:dbv:threshold_and_plateau}
For $\chi\neq \chi_\mathrm{dbv}$, the drift has the following sign structure on $\alpha\in(0,1]$.
\begin{enumerate}[leftmargin=*,label=(\arabic*)]
    \item If $0<\chi<\chi_\mathrm{dbv}$, then $D_\chi(\alpha)>0$ for all $\alpha\in(0,1]$.
    \item If $\chi>\chi_\mathrm{dbv}$, then $\alpha^*(\chi)\in(0,1)$ and
    \begin{equation*}
        D_\chi(\alpha)<0\ \text{ for }\ \alpha\in(0,\alpha^*(\chi)),
        \quad
        D_\chi(\alpha)>0\ \text{ for }\ \alpha\in(\alpha^*(\chi),1].
    \end{equation*}
\end{enumerate}
In addition, $\alpha^*(\chi)\to \frac12$ as $\chi\to\infty$.
\end{corollary}

\begin{proof}
Fix $\chi>0$ and write $B(\chi)\coloneqq 2-2\chi-2e^{-\chi}$.
Then \Cref{eq:dbv:drift_closed_form} reads
\begin{equation*}
    D_\chi(\alpha)=\alpha\bigl(A(\chi)-\alpha B(\chi)\bigr).
\end{equation*}
For all $\chi>0$ we have $B(\chi)<0$, so the map $\alpha\mapsto A(\chi)-\alpha B(\chi)$ is strictly increasing on $[0,1]$.
If $A(\chi)>0$ (i.e.\ $\chi<\chi_\mathrm{dbv}$), then $A(\chi)-\alpha B(\chi)\ge A(\chi)>0$ for all $\alpha\in[0,1]$,
and thus $D_\chi(\alpha)>0$ for all $\alpha\in(0,1]$.

If $A(\chi)<0$ (i.e.\ $\chi>\chi_\mathrm{dbv}$), then the increasing function $\alpha\mapsto A(\chi)-\alpha B(\chi)$ has a unique root
\begin{equation*}
    \alpha^*(\chi)=\frac{A(\chi)}{B(\chi)}=\frac{2-\chi-2e^{-\chi}}{2-2\chi-2e^{-\chi}}.
\end{equation*}
Since both $A(\chi)$ and $B(\chi)$ are negative, we have $\alpha^*(\chi)>0$.
Moreover $\alpha^*(\chi)<1$ is equivalent to $A(\chi)>B(\chi)$, i.e.\ $\chi>0$.
Thus $\alpha^*(\chi)\in(0,1)$ and the sign change of $A(\chi)-\alpha B(\chi)$ yields the stated sign pattern for $D_\chi(\alpha)$ on $(0,1]$.

For $\chi\to\infty$, since $e^{-\chi}\to 0$, we have $\alpha^*(\chi)\to (2-\chi)/(2-2\chi)=\frac12$.
\end{proof}

\paragraph{Interpretation.}
At state $y=\alpha n$, the fraction of correct bits equals $1-\alpha$.
Thus for $\chi>\chi_\mathrm{dbv}$, the drift points away from the optimum throughout a whole neighbourhood of the optimum,
namely for all sufficiently small $\alpha$, up to the plateau $\alpha^*(\chi)$.
Equivalently, the process is pushed towards a constant fraction of wrong bits, and correspondingly stabilises around a constant fraction of correct bits $1-\alpha^*(\chi)$.

\subsection{Finite-$n$ Drift and Finite-$n$ Plateau}
\label{sec:dbv:finite_n}

A convenient property of DBV is that the exact drift admits a simple closed form for every finite~$n$.
This allows us to obtain a uniform multiplicative drift bound in the subcritical regime and quantify how the finite-$n$ plateau differs from $\alpha^*(\chi)$ in the approximated setting.

\paragraph{Finite-$n$ Analogue of $e^{-\chi}$.}
Let $K_n\sim\Bin(n,\chi/n)$ be the number of flipped bits in one mutation step and define
\begin{equation*}
    r_n(\chi)\coloneqq \Prob{K_n=0}
    =\Bigl(1-\frac{\chi}{n}\Bigr)^{\!n}.
\end{equation*}
In the continuous approximation we have $\Prob{\Poi(\chi)=0}=e^{-\chi}$. Thus $r_n(\chi)$ is the natural finite-$n$ analogue of $e^{-\chi}$.

\begin{proposition}
\label{prop:dbv:discrete_drift_clean}
Fix $\chi>0$ and $n\ge 2$.
For DBV, for all $y\in\{0,1,\dots,n\}$ with $\alpha\coloneqq y/n$,
\begin{equation}
\label{eq:dbv:discrete_drift_clean}
    \Delta(y)
    = \frac{n}{n-1}\,
    \alpha\Bigl(\bigl(2-\chi-2r_n(\chi)-\tfrac{\chi}{n}\bigr)
    -\alpha\bigl(2-2\chi-2r_n(\chi)\bigr)\Bigr).
\end{equation}
\end{proposition}

\begin{proof}
The proof is a direct computation and is deferred to \Cref{sec:app:dbv:discrete_drift_clean}.
\end{proof}

\paragraph{Roots.}
The expression \Cref{eq:dbv:discrete_drift_clean} is a quadratic polynomial in $\alpha$ and is already factorised.
In particular, $\alpha=0$ is always a root.
The second root is obtained by solving the bracket in \Cref{eq:dbv:discrete_drift_clean} for~$\alpha$, namely
\begin{equation}
\label{eq:dbv:alpha_n_star}
    \alpha_n^*(\chi)
    \coloneqq
    \frac{2-\chi-2(1-\chi/n)^n-\chi/n}{2-2\chi-2(1-\chi/n)^n}.
\end{equation}

Note that for every $\chi>0$ and all sufficiently large $n$ we have $2-2\chi-2(1-\chi/n)^n<0$,
so the bracket in \Cref{eq:dbv:discrete_drift_clean} is strictly increasing in~$\alpha$.

\begin{corollary}
\label{cor:dbv:finite_n_mult_drift}
Fix $\chi\in(0,\chi_\mathrm{dbv})$.
Then for all sufficiently large $n$ and all $y\in\{1,\dots,n\}$,
\begin{equation*}
    \Delta(y)\ge \frac{A(\chi)}{2n}\,y.
\end{equation*}
\end{corollary}

\begin{proof}
Recall $\alpha=y/n$ and $r_n(\chi)=(1-\chi/n)^n$.
From \Cref{eq:dbv:discrete_drift_clean},
\begin{equation*}
    \Delta(y)=\frac{n}{n-1}\,\alpha\Bigl(\bigl(2-\chi-2r_n(\chi)-\tfrac{\chi}{n}\bigr)
      -\alpha\bigl(2-2\chi-2r_n(\chi)\bigr)\Bigr).
\end{equation*}
For fixed $\chi>0$ we have $r_n(\chi)\to e^{-\chi}$ as $n\to\infty$, hence
\begin{equation*}
    2-2\chi-2r_n(\chi)\to 2-2\chi-2e^{-\chi}<0.
\end{equation*}
Therefore, for all sufficiently large $n$ we have $2-2\chi-2r_n(\chi)<0$, and thus
\begin{equation*}
    -\alpha\bigl(2-2\chi-2r_n(\chi)\bigr)\ge 0.
\end{equation*}
Consequently,
\begin{equation*}
    \Delta(y)\ge \frac{n}{n-1}\,\alpha\bigl(2-\chi-2r_n(\chi)-\tfrac{\chi}{n}\bigr)
    \ge \alpha\bigl(2-\chi-2r_n(\chi)-\tfrac{\chi}{n}\bigr).
\end{equation*}

Moreover, $r_n(\chi)\le e^{-\chi}$ implies $2-\chi-2r_n(\chi)\ge 2-\chi-2e^{-\chi}$.
Choose $n_0$ such that $\chi/n\le (2-\chi-2e^{-\chi})/2$ for all $n\ge n_0$.
Then for $n\ge n_0$,
\begin{equation*}
    \Delta(y)\ge \alpha\,\frac{2-\chi-2e^{-\chi}}{2}
    =\frac{2-\chi-2e^{-\chi}}{2n}\,y.
\end{equation*}
\end{proof}

\paragraph{The Difference between $\alpha^*_n(\chi)$ and $\alpha^*(\chi)$}
For $\chi>\chi_\mathrm{dbv}$, both the limiting plateau $\alpha^*(\chi)$ from \Cref{eq:dbv:alpha_star_cont}
and the finite-$n$ plateau $\alpha_n^*(\chi)$ from \Cref{eq:dbv:alpha_n_star} are well-defined.
The next lemma shows that $\alpha_n^*(\chi)=\alpha^*(\chi)\pm \Oh{1/n}$.


\begin{corollary}
\label{cor:dbv:alpha_star_shift}
For every fixed $\chi>\chi_\mathrm{dbv}$ we have $\alpha^*_n(\chi) = \alpha^*(\chi) \pm \Oh{1/n}$.
\end{corollary}

\begin{proof}
    By \Cref{eq:dbv:alpha_star_cont} and \Cref{eq:dbv:alpha_n_star}, we may write
    \begin{equation*}
         \alpha^*(\chi) = \frac{a}{b},\qquad\alpha_n^*(\chi) = \frac{a_n}{b_n},
    \end{equation*}
    where 
    \begin{equation*}
        a \coloneqq 2 - \chi - 2e^{-\chi},\qquad b \coloneqq 2 - 2\chi - 2e^{-\chi},
    \end{equation*}
    and
    \begin{equation*}
        a_n \coloneqq 2 - \chi - 2(1-\chi/n)^n-\chi/n,\qquad b_n \coloneqq 2 - 2\chi - 2(1 - \chi/n)^n.
    \end{equation*}
    By \cite[Corollary~1.4.6]{doerr2019probabilistic}, for all $n\ge \chi$,
    \begin{equation*}
        \left(1-\frac{\chi}{n}\right)^n \le e^{-\chi}\le \left(1-\frac{\chi}{n}\right)^{n-\chi},
    \end{equation*}
    which implies, for fixed $\chi$,
    \begin{equation*}
        \Big(1-\frac{\chi}{n}\Big)^n = e^{-\chi} + \Oh{1/n}.
    \end{equation*}
    Therefore
    \begin{equation}\label{eq:dbv:num_denom_diff}
        a_n - a = \Oh{1/n}\qquad\text{and}\qquad b_n - b = \Oh{1/n}.
    \end{equation}
    Since $\chi > \chi_{\mathrm{dbv}}$, we have $a = 2 - \chi - 2e^{-\chi}<0$ and hence $b = a - \chi < 0$. As $b_n \to b$ for $n\to \infty$, we have $\abs{b_n}\ge \abs{b}/2$ for sufficiently large $n$.
    Finally,
    \begin{equation*}
        \alpha_n^*(\chi) - \alpha^*(\chi) = \frac{a_n}{b_n} - \frac{a}{b} = \frac{(a_n - a)b - a(b_n - b)}{bb_n}.
    \end{equation*}
    Using \Cref{eq:dbv:num_denom_diff} and the fact that $bb_n$ is bounded away from zero for all sufficiently large $n$, it follows that $\alpha_n^*(\chi) = \alpha^*(\chi) \pm \Oh{1/n}$.
\end{proof}

\subsection{Proof of \Cref{thm:dbv:main}}
\label{sec:dbv:proof}

\begin{proof}
Case $0<\chi<\chi_\mathrm{dbv}$.
Fix $\chi\in(0,\chi_\mathrm{dbv})$.
By \Cref{cor:dbv:finite_n_mult_drift}, for all sufficiently large $n$,
\begin{equation*}
    \E{Y_t-Y_{t+1}\mid Y_t=y}=\Delta(y)\ge \frac{A(\chi)}{2n}\,y
    \quad\text{for all }y\in\{1,\dots,n\}.
\end{equation*}
Thus \Cref{thm:pre:mult_drift_upper_tail} applies with $\delta=A(\chi)/(2n)>0$ and $s_{\min}=1$.
This yields $T=\Oh{n\log n}$ a.a.s.\ and $\E{T}=\Oh{n\log n}$.

Case $\chi>\chi_\mathrm{dbv}$. Fix $\chi>\chi_\mathrm{dbv}$ and let $\alpha^*(\chi)$ be as in \Cref{eq:dbv:alpha_star_cont}. 
Since the drift is positive for all $\alpha > \alpha^*(\chi)$, the $\opoEA$ reaches distance $\alpha > \alpha^*(\chi)$ in time $\Oh{n}$.

So let us consider constants $0<a<b<\alpha^*(\chi)$ and show that it takes exponential time to reach distance $an$ from the optimum. 
By \Cref{cor:dbv:threshold_and_plateau}, we have $D_\chi(\alpha)<0$ for all $\alpha\in(0,\alpha^*(\chi))$.
Hence, choosing $0<a<b<\alpha^*(\chi)$, continuity yields some $\eps>0$ such that
\begin{equation}
\label{eq:dbv:neg_interval}
    D_\chi(\alpha)\le -2\eps\quad\text{for all }\alpha\in[a,b].
\end{equation}
Applying \Cref{lem:pre:simplified_drift_conditions} yields that, for all sufficiently large $n$,
the hypotheses of \Cref{thm:pre:simplified_drift} hold for $(Y_t)_{t\ge 0}$ on the interval $[a_n,b_n]$ with
$a_n\coloneqq\lfloor an\rfloor$ and $b_n\coloneqq\lfloor bn\rfloor$. Indeed, condition~(1) holds with $\eps$ as above, and condition~(2) holds with $\delta=1$ and $r(\ell)=e^\chi$, where $\ell\coloneqq b_n - a_n$. Define the hitting time $T^*\coloneqq \min\SetBuilder{t\ge 0}{Y_t\le a_n}$.
Then \Cref{thm:pre:simplified_drift} implies that there exists a constant $c^*>0$ such that
\begin{equation*}
    \Prob{T^* \le 2^{c^*\ell/r(\ell)} \mid Y_0\ge b_n}
    = 2^{-\Om{\ell/r(\ell)}}
    = 2^{-\Om{n}},
\end{equation*}
for all sufficiently large $n$, where we used $r(\ell)=e^\chi\le\Oh{1}$ and $\ell \le (b-a)n\le\Oh{n}$.
Equivalently, with probability $1-2^{-\Om{n}}$ we have $T^*\ge 2^{\Om{n}}$ on the event $\{Y_0\ge b_n\}$.
Since the runtime $T$ satisfies $T\ge T^*$, the same lower bound holds for~$T$.

Under u.a.r.\ initialisation we have $Y_0\sim\Bin(n,1/2)$, and since $b<1/2$ is fixed, a Chernoff bound yields
$\Prob{Y_0\ge b_n}=1-e^{-\Om{n}}$.
Therefore $T=2^{\Om{n}}$ a.a.s.
Moreover,
\begin{equation*}
    \E{T}\ge \E{T\mid Y_0\ge b_n}\Prob{Y_0\ge b_n}\ge 2^{\Om{n}}(1-e^{-\Om{n}})=2^{\Om{n}}.
\end{equation*}
\end{proof}
\section{Uniform Weight Model}
\label{sec:uniform}

In this chapter we analyse the $\opoEA$ in the Uniform weight model, where in each generation $t$ the fitness function is
\begin{equation*}
    f^{(t)}(x)=\sum_{i=1}^n W_i^{(t)}x_i
    \quad\text{with}\quad
    W_1^{(t)},\dots,W_n^{(t)} \overset{\text{i.i.d.}}{\sim} \Unif(0,1).
\end{equation*}
As in the other cases, the weights are positive, so the unique optimum is $(1,\dots,1)$.
As in \Cref{sec:preliminaries}, we track progress via the Markov chain $(Y_t)_{t\ge 0}$, where $Y_t$ denotes the number of zero-bits of the current search point.

\paragraph{Motivation and Results.}
Vermetten~et~al.\ \cite{vermettenEmpiricalAnalysisDynamic2024} observe a stark difference between the Uniform weight model and the other variants:
for small mutation rates all considered settings behave similarly, since progress is dominated by one-bit mutations---a single flipped bit is accepted iff it flips a zero-bit.
For larger mutation rates, however, the Uniform weight model becomes visibly easier than the rest.
Our results identify a critical value $\chi_\mathrm{unif}\approx 2.76531$ such that, for every fixed $\chi<\chi_\mathrm{unif}$,
the $\opoEA$ optimises the uniform setting in $\Oh{n\log n}$ in expectation and a.a.s.
This is consistent with \cite[Figure~3]{vermettenEmpiricalAnalysisDynamic2024}, which suggests an intermediate range between
$\chi_\mathrm{dbv}\approx 1.59362$ and $\chi_\mathrm{unif}\approx 2.76531$ in which the optimum is reached efficiently for $\Unif(0,1)$ weights, while the other variants fail to be optimised.

\begin{theorem}
\label{thm:unif:main}
Let $\chi_\mathrm{unif} \approx 2.76531$ be the unique strictly positive root of $g(\chi)$ from \Cref{lem:unif:right_derivative}.
Consider the $\opoEA$ with mutation parameter $\chi$ on dynamic linear functions in the Uniform weight model.

For every fixed $0<\chi<\chi_\mathrm{unif}$ the runtime is $\Oh{n\log n}$ in expectation and a.a.s.
On the other hand, for every fixed $\chi>\chi_\mathrm{unif}$ the runtime is $2^{\Om{n}}$ in expectation and a.a.s.
\end{theorem}

\paragraph{Roadmap.}
We proceed in four steps.
First, we derive the conditional acceptance probability $p_A(k,i)$ for uniform weights (\Cref{prop:unif:irwin_hall}), yielding an explicit expression for the continuous drift approximation $D(\chi,\alpha)$.
Second, we study the regime near the optimum by expanding $D_\chi(\alpha)$ at $\alpha=0$.
This identifies the critical value $\chi_\mathrm{unif}$ through the sign of the right derivative $\partial_+D_\chi(0)=g(\chi)$ (\Cref{lem:unif:right_derivative} and \Cref{cor:unif:small_alpha_sign}).
Third, for $\chi<\chi_\mathrm{unif}$ we lift this local positivity to all $\alpha\in(0,1]$ by exploiting convexity of $\alpha\mapsto D_\chi(\alpha)$ (\Cref{lem:unif:convexity} and \Cref{cor:unif:global_positivity}).
Finally, we transfer the resulting bounds $D_\chi$ to the drift $\Delta$ via the drift tools from \Cref{sec:preliminaries}.
We apply multiplicative drift (together with an argument for the drift below a constant threshold) in the subcritical regime $\chi<\chi_\mathrm{unif}$, and the Simplified Drift Theorem on a constant interval in the supercritical regime $\chi>\chi_\mathrm{unif}$.

\subsection{Acceptance Probability}

\begin{proposition}
\label{prop:unif:irwin_hall}
In the Uniform weight model the acceptance probability defined in \Cref{eq:pre:acceptance_probability} satisfies, for all integers $k > 0$ and $0\le i\le k$,
\begin{equation*}
    p_A(k,i)
    = \frac{1}{k!}\sum_{m=0}^{i}(-1)^m\binom{k}{m}(i-m)^{k}.
\end{equation*}
\end{proposition}

\begin{proof}
A full derivation is given in \Cref{app:prop:unif:irwin_hall}.
\end{proof}

\paragraph{Interpretation.}
Conditional on $(K=k,I=i)$, acceptance means that the total weight gained from the $i$ flipped zero-bits exceeds the total weight lost from the $k-i$ flipped one-bits.
For the Uniform weight model this comparison can be rewritten using a sum of $k$ i.i.d.\ $\Unif(0,1)$ variables, hence an Irwin--Hall cdf appears (see \Cref{sec:preliminaries}).
In particular, the special case $i=1$ yields $p_A(k,1)=1/k!$, which will make the linear term of $D_\chi(\alpha)$ at $\alpha=0$ explicit in \Cref{lem:unif:right_derivative}.

In the next section we use the local expansion of $D_\chi(\alpha)$ at $\alpha=0$ to identify the critical mutation parameter $\chi_\mathrm{unif}$.

\subsection{Local Threshold}

We first analyse the drift close to the optimum, i.e.\ for $\alpha\approx 0$.
In this regime only few zero-bits remain, and the sign of the drift is determined by the first-order term of $\alpha\mapsto D_\chi(\alpha)$ at $\alpha=0$.
Since $\alpha=0$ is a boundary point, we work with the right derivative.

\begin{lemma}
\label{lem:unif:right_derivative}
In the Uniform weight model the right derivative of the drift at $\alpha=0$ satisfies
\begin{equation*}
    \partial_+ D_\chi(0)=g(\chi),
    \quad
    g(\chi)\coloneqq \sum_{k=0}^\infty p_k(\chi)\,k(2-k)\,\frac{1}{k!},
\end{equation*}
where $p_k(\chi)\coloneqq e^{-\chi}\chi^k/k!$ denotes the Poisson pmf.
The function $g$ has a unique strictly positive root $\chi_\mathrm{unif}\approx 2.76531$ and satisfies
$g(\chi)>0$ for $\chi\in(0,\chi_\mathrm{unif})$ and $g(\chi)<0$ for $\chi\in(\chi_\mathrm{unif},\infty)$.
\end{lemma}

\begin{proof}
We compute $\partial_+D_\chi(0)$ via \Cref{prop:pre:right_derivative} and then analyse $g(\chi)$ using a Bessel-function representation.
The full argument (including uniqueness of the root and the sign change) is given in \Cref{app:lem:unif:right_derivative}.
\end{proof}

\begin{corollary}
\label{cor:unif:small_alpha_sign}
In the Uniform weight model and for any fixed $\chi\in(0,\infty)$ with $\chi\neq \chi_\mathrm{unif}$, the drift $D_\chi(\alpha)$ has the same sign as $g(\chi)$ for all sufficiently small $\alpha>0$.
In particular, for sufficiently small $\alpha$ we have $D_\chi(\alpha)>0$ if $\chi\in(0,\chi_\mathrm{unif})$ and $D_\chi(\alpha)<0$ if $\chi\in(\chi_\mathrm{unif},\infty)$.
\end{corollary}

\begin{proof}
Fix $\chi > 0$. By \Cref{lem:unif:right_derivative} we have $D_\chi(\alpha)/\alpha\to g(\chi)$ as $\alpha\to 0^+$.
By the definition of the limit, there exists $\alpha^*>0$ such that for all $\alpha\in(0,\alpha^*)$,
\begin{equation*}
    \abs*{\frac{D_\chi(\alpha)}{\alpha}-g(\chi)}\le \frac12\abs{g(\chi)}.
\end{equation*}
If $g(\chi)>0$, then $D_\chi(\alpha)/\alpha\ge g(\chi)-\frac12 g(\chi)=\frac12 g(\chi)$ and hence $D_\chi(\alpha)\ge \frac12\alpha\,g(\chi)>0$ for all $\alpha\in(0,\alpha^*)$.
If $g(\chi)<0$, then $D_\chi(\alpha)/\alpha\le g(\chi)+\frac12\abs{g(\chi)} = g(\chi)-\frac12 g(\chi)=\frac12 g(\chi)$ and hence $D_\chi(\alpha)\le \frac12\alpha\,g(\chi)<0$ for all $\alpha\in(0,\alpha^*)$.

Finally, \Cref{lem:unif:right_derivative} gives $g(\chi)>0$ for $\chi\in(0,\chi_\mathrm{unif})$ and $g(\chi)<0$ for $\chi\in(\chi_\mathrm{unif},\infty)$, which completes the proof.
\end{proof}

The corollary determines the sign of $D_\chi(\alpha)$ for all sufficiently small $\alpha>0$ and thus identifies $\chi_\mathrm{unif}$ as a local threshold.
To obtain a runtime bound in the regime $\chi<\chi_\mathrm{unif}$, we need a lower bound that holds uniformly for all $\alpha\in[0,1]$.
Since $D_\chi(0)=0$ (as noted in \Cref{sec:preliminaries}) and $\partial_+D_\chi(0)=g(\chi)>0$ for $\chi\in(0,\chi_\mathrm{unif})$, convexity of $\alpha\mapsto D_\chi(\alpha)$ on $[0,1]$ would imply
\begin{equation*}
    D_\chi(\alpha)\ge \alpha\,\partial_+D_\chi(0)=g(\chi)\alpha
    \quad\text{for all }\alpha\in[0,1].
\end{equation*}
We establish this convexity next.

\subsection{Global Positivity via Convexity}

For the convexity proof it is convenient to separate the contribution of a fixed mutation size $k$, which leads to the functions $H_k$ below.

\paragraph{Decomposition by Mutation Size.}
Recall \Cref{eq:pre:approximate_drift} and write $p_k(\chi)\coloneqq e^{-\chi}\chi^k/k!$ for the Poisson pmf and $b_{k,i}(\alpha)\coloneqq \binom{k}{i}\alpha^i(1-\alpha)^{k-i}$ for the binomial pmf. Moreover, for each $k\in\N_0$ define the function
\begin{equation*}
    H_k(\alpha)\coloneqq \sum_{i=0}^{k} h_k(i)\,b_{k,i}(\alpha),
    \quad\text{where}\quad
    h_k(i)\coloneqq (2i-k)\,p_A(k,i).
\end{equation*}
Then, for all $\chi>0$ and $\alpha\in[0,1]$, we may write
\begin{equation*}
    D(\chi,\alpha)=\sum_{k=0}^{\infty} p_k(\chi)\,H_k(\alpha).
\end{equation*}

\begin{lemma}
\label{lem:unif:convexity_tail}
In the Uniform weight model and for all $\chi\in(0,\chi_\mathrm{unif})$ there exists a constant $c>0$ such that
\begin{equation*}
    \abs*{\sum_{k = 10}^\infty p_k(\chi)H''_k(\alpha)}
    \le c\,p_{10}(\chi).
\end{equation*}
\end{lemma}

\begin{proof}
For fixed $\chi>0$ and all $k\ge 10$ we have
\begin{equation*}
    \frac{p_{k+1}(\chi)}{p_k(\chi)}=\frac{\chi}{k+1}\le \frac{\chi}{11}.
\end{equation*}
Telescoping this inequality yields
\begin{equation}
\label{eq:unif:pois_tail_geom}
    p_k(\chi)\le p_{10}(\chi)\Bigl(\frac{\chi}{11}\Bigr)^{k-10}\quad\text{for all }k\ge 10.
\end{equation}

For $k\ge 2$ and through the identity $b'_{k,i}(\alpha)=k\bigl(b_{k-1,i-1}(\alpha)-b_{k-1,i}(\alpha)\bigr)$ the second derivative of $H_k$ yields
\begin{equation}
\label{eq:unif:second_deriv_diff}
    H''_k(\alpha)
    = k(k-1)\sum_{i=0}^{k-2}\bigl(h_k(i+2)-2h_k(i+1)+h_k(i)\bigr)\,b_{k-2,i}(\alpha).
\end{equation}

Fix $k\ge 10$ and $0\le i\le k-2$ and let $q_1\coloneqq p_A(k,i+1)-p_A(k,i)$ and $q_2\coloneqq p_A(k,i+2)-p_A(k,i+1)$.
Then $0\le q_1,q_2\le 1$ and $\abs{q_2-q_1}\le 1$.
A direct computation gives
\begin{equation*}
    h_k(i+2)-2h_k(i+1)+h_k(i)=(2i-k)(q_2-q_1)+4q_2,
\end{equation*}
hence, using $\abs{2i-k}\le k$,
\begin{equation*}
    \abs{(2i-k)(q_2-q_1)+4q_2}\le k+4\le \frac{7}{5}k.
\end{equation*}
Inserting into \Cref{eq:unif:second_deriv_diff} and using $\sum_{i=0}^{k-2} b_{k-2,i}(\alpha)=1$ yields
\begin{equation}
\label{eq:unif:second_deriv_bound_k}
    \abs{H''_k(\alpha)}\le \frac{7}{5}k^3
    \quad\text{for all }k\ge 10.
\end{equation}

Combining \Cref{eq:unif:pois_tail_geom} and \Cref{eq:unif:second_deriv_bound_k} yields
\begin{align*}
    \abs*{\sum_{k=10}^{\infty} p_k(\chi)H''_k(\alpha)}
    \le \sum_{k=10}^{\infty} p_k(\chi)\,\abs*{H''_k(\alpha)}
    &\le \frac{7}{5}\,p_{10}(\chi)\sum_{k=10}^{\infty} k^3\Bigl(\frac{\chi}{11}\Bigr)^{k-10} \\
    &= \frac{7}{5}\,p_{10}(\chi)\sum_{j=0}^{\infty} (j+10)^3\Bigl(\frac{\chi}{11}\Bigr)^{j}.
\end{align*}
The series on the right hand side has strictly positive coefficients and is therefore increasing in $\chi$ for $\chi\in(0,11)$.
Since $\chi_\mathrm{unif}<14/5<11$, we may bound it for all $\chi\in(0,\chi_\mathrm{unif})$ by its value at $\chi=14/5$.
Evaluating the resulting closed form at $\chi=14/5$ gives
\begin{equation*}
    \sum_{j=0}^{\infty} (j+10)^3\Bigl(\frac{14}{55}\Bigr)^{j}\approx 1503.725,
\end{equation*}
hence the bound holds with $c\coloneqq \frac{7}{5}\,1504$.
\end{proof}

The next step is to prove convexity of $\alpha\mapsto D_\chi(\alpha)$ for $\chi<\chi_\mathrm{unif}$.
We do this by writing $D(\chi,\alpha)=\sum_{k\ge 0} p_k(\chi)H_k(\alpha)$ and controlling the tail contribution $\sum_{k\ge 10}p_k(\chi)H_k''(\alpha)$ via \Cref{lem:unif:convexity_tail}, while handling the finitely many terms $k\le 9$ by explicit lower bounds.

\begin{lemma}
\label{lem:unif:convexity}
In the Uniform weight model there exists a polynomial $P(\chi)$ such that for all $(\alpha,\chi)\in[0,1]\times(0,\chi_\mathrm{unif})$,
\begin{equation*}
    \frac{\partial^2}{\partial\alpha^2}D(\chi,\alpha)\ge e^{-\chi}P(\chi)\ge 0.
\end{equation*}
In particular, for each fixed $\chi\in(0,\chi_\mathrm{unif})$ the function $\alpha\mapsto D_\chi(\alpha)$ is convex on $[0,1]$.
\end{lemma}

\begin{proof}
By \Cref{prop:unif:irwin_hall} we have $p_A(k,i)=F_k(i)$ for integers $0\le i\le k$,
where $F_k$ denotes the Irwin--Hall cdf. In particular, $p_A(k,i)$ is independent of $\alpha$ and $\chi$.
Recalling the definition
\begin{equation*}
    H_k(\alpha)=\sum_{i=0}^{k} h_k(i)\,b_{k,i}(\alpha),
    \quad h_k(i)=(2i-k)p_A(k,i),
\end{equation*}
it follows that each $H_k$ is a polynomial in $\alpha$ of degree at most $k$, and hence $H_k''$ is a polynomial of degree at most $k-2$.
Moreover, differentiating the identity
\begin{equation*}
    D(\chi,\alpha)=\sum_{k=0}^{\infty} p_k(\chi)\,H_k(\alpha)
\end{equation*}
twice with respect to $\alpha$ yields
\begin{equation}
\label{eq:unif:D_second_derivative_series}
    \frac{\partial^2}{\partial\alpha^2}D(\chi,\alpha)
    =\sum_{k=0}^{\infty} p_k(\chi)\,H_k''(\alpha).
\end{equation}

By \Cref{lem:unif:convexity_tail} there exists a constant $c>0$ such that for all $\chi\in(0,\chi_\mathrm{unif})$,
\begin{equation*}
    \abs*{\sum_{k=10}^{\infty} p_k(\chi)\,H_k''(\alpha)}\le c\,p_{10}(\chi)
    \quad\text{for all }\alpha\in[0,1].
\end{equation*}
Consequently, for all $(\alpha,\chi)\in[0,1]\times(0,\chi_\mathrm{unif})$,
\begin{equation}
\label{eq:unif:D_second_derivative_split}
    \frac{\partial^2}{\partial\alpha^2}D(\chi,\alpha)
    \ge \sum_{k=0}^{9} p_k(\chi)\,H_k''(\alpha)-c\,p_{10}(\chi).
\end{equation}

For each $0\le k\le 9$ choose constants $m_k$ such that
\begin{equation*}
    m_k \le H_k''(\alpha)\quad\text{for all }\alpha\in[0,1],
\end{equation*}
and set $m_{10}\coloneqq -c$. Using $p_k(\chi)=e^{-\chi}\chi^k/k!$, inequality \Cref{eq:unif:D_second_derivative_split} implies
\begin{equation*}
    \frac{\partial^2}{\partial\alpha^2}D(\chi,\alpha)
    \ge e^{-\chi}\sum_{k=0}^{10} m_k\frac{\chi^k}{k!}
    = e^{-\chi}P(\chi),
\end{equation*}
where we define the polynomial
\begin{equation*}
    P(\chi)\coloneqq \sum_{k=0}^{10} m_k\frac{\chi^k}{k!}.
\end{equation*}

It remains to show $P(\chi)\ge 0$ for all $\chi\in(0,\chi_\mathrm{unif})$.
We proceed by a computer-assisted verification.
First, we compute valid constants $m_k$ for $0\le k\le 9$ and take $m_{10}=-c$, where $c$ is the tail constant from \Cref{lem:unif:convexity_tail}.
Second, on the interval $[0,14/5]\supseteq(0,\chi_\mathrm{unif})$ we verify $P(0)=0$ and $P(14/5)>0$, and we exclude any further real roots of $P$ in $[0,14/5]$.
Hence $P(\chi)\ge 0$ for all $\chi\in(0,\chi_\mathrm{unif})$. The full verification (choice of $m_k$ and root exclusion for $P$ on $[0,14/5]$) is documented in a reproducible script, available on our Zenodo repository \cite{hasebe_2026_19232932}.

Therefore $e^{-\chi}P(\chi)\ge 0$ on $(0,\chi_\mathrm{unif})$, and the claimed inequality follows.
In particular, for each fixed $\chi\in(0,\chi_\mathrm{unif})$ we have $D''_\chi(\alpha)\ge 0$ for all $\alpha\in[0,1]$,
so $\alpha\mapsto D_\chi(\alpha)$ is convex on $[0,1]$.
\end{proof}

Convexity turns the local slope at $\alpha=0$ into a global bound: since $D_\chi(0)=0$ and $D_\chi$ is convex, we obtain
$D_\chi(\alpha)\ge \alpha\,\partial_+D_\chi(0)=g(\chi)\alpha$ for all $\alpha\in[0,1]$ when $\chi<\chi_\mathrm{unif}$.

\begin{corollary}
\label{cor:unif:global_positivity}
Let $g(\chi)$ be as in \Cref{lem:unif:right_derivative}. Then for any fixed $\chi\in(0,\chi_\mathrm{unif})$ and all $\alpha\in(0,1]$,
\begin{equation*}
    D_\chi(\alpha)\ge g(\chi)\alpha > 0.
\end{equation*}
\end{corollary}

\begin{proof}
Fix $\chi\in(0,\chi_\mathrm{unif})$. By \Cref{lem:unif:convexity}, the function $\alpha\mapsto D_\chi(\alpha)$ is convex on $[0,1]$.
Hence it lies above its tangent at every point and at $\alpha=0$ we have
\begin{equation}
\label{eq:unif:supporting_line_at_0}
    D_\chi(\alpha)\ge D_\chi(0)+\alpha\,\partial_+ D_\chi(0)\quad\text{for all }\alpha\in[0,1].
\end{equation}
At $\alpha=0$ the algorithm is at the optimum, so no offspring will be accepted and therefore $D_\chi(0)=0$.
Moreover, by \Cref{lem:unif:right_derivative} we have $\partial_+D_\chi(0)=g(\chi)$, and $g(\chi)>0$ for $\chi\in(0,\chi_\mathrm{unif})$.
Substituting these facts into \Cref{eq:unif:supporting_line_at_0} yields $D_\chi(\alpha)\ge \alpha\,g(\chi)> 0$ for all $\alpha\in(0,1]$.
\end{proof}

\subsection{Proof of \Cref{thm:unif:main}}
\label{sec:unif:proof}

\begin{proof}
Case $0<\chi<\chi_\mathrm{unif}$. Fix $\chi \in (0,\chi_\mathrm{unif})$.
By \Cref{cor:unif:global_positivity} we have for all $\alpha\in[0,1]$,
\begin{equation}
\label{eq:unif:global_positivity}
    D_\chi(\alpha)\ge g(\chi)\alpha,
    \quad\text{where }g(\chi)>0.
\end{equation}
By \Cref{prop:pre:drift_convergence}, there exists a constant $c(\chi)>0$ such that, uniformly in $\alpha\in[0,1]$,
\begin{equation}
\label{eq:unif:drift_convergence}
    \abs*{\Delta(\lfloor \alpha n\rfloor)-D_\chi(\alpha)}\le \frac{c(\chi)}{n}.
\end{equation}
Let $y_0\coloneqq \bigl\lceil 2c(\chi)/g(\chi)\bigr\rceil$.
Then for all $y\ge y_0$, using \Cref{eq:unif:global_positivity} and \Cref{eq:unif:drift_convergence} we obtain
\begin{equation}
\label{eq:unif:mult_drift_lower}
    \Delta(y)\ge g(\chi)\frac{y}{n}-\frac{c(\chi)}{n}\ge\frac{g(\chi)}{2n}\,y.
\end{equation}

Define the runtime $T_{y_0}\coloneqq \min\SetBuilder{t\ge 0}{Y_t\le y_0}$ and the shifted process
\begin{equation*}
    Z_t \coloneqq \max\{Y_t-y_0,0\}.
\end{equation*}
Then $Z_t\in\{0,1,\dots,n-y_0\}$ and $T_{y_0}=\min\SetBuilder{t\ge 0}{Z_t=0}$.
Moreover, for all $s\ge 1$,
\begin{equation*}
    \E{Z_t-Z_{t+1}\mid Z_t=s}
    = \Delta(s+y_0)
    \ge \frac{g(\chi)}{2n}(s+y_0)
    \ge \frac{g(\chi)}{2n}\,s,
\end{equation*}
where we used \Cref{eq:unif:mult_drift_lower}.
Hence \Cref{thm:pre:mult_drift_upper_tail} applies to $(Z_t)_{t\ge 0}$ and yields $T_{y_0}=\Oh{n\log n}$ a.a.s.
Since \Cref{thm:pre:mult_drift_upper_tail} provides an exponential upper tail, it also implies
\begin{equation}
\label{eq:unif:runtime_y0}
    \E{T_{y_0}} = \Oh{n\log n}.
\end{equation}

To cover the remaining states $\{1,\dots,y_0\}$ we invoke \Cref{lem:pre:constant_drift}.
Its assumptions are satisfied since all weights are positive almost surely and the multiplicative drift lower bound
\Cref{eq:unif:mult_drift_lower} holds for all $y\ge y_0$.
Consequently, for all $n$ sufficiently large,
\begin{equation}
\label{eq:unif:endgame_exp}
    \sup_{1\le y\le y_0}\E{T\mid Y_0=y}=\Oh{n},
\end{equation}
and for every fixed $y\in\{1,\dots,y_0\}$,
\begin{equation}
\label{eq:unif:endgame_tail}
    \Prob{T>n\log n\mid Y_0=y}=\oh{1}.
\end{equation}

We now combine the two parts. Recall $T\coloneqq\min\SetBuilder{t\ge 0}{Y_t=0}$ and $T_{y_0}\coloneqq\min\SetBuilder{t\ge 0}{Y_t\le y_0}$.
At time $T_{y_0}$ we have $Y_{T_{y_0}}\in\{0,1,\dots,y_0\}$.
Since $(Y_t)_{t\ge 0}$ is a time-homogeneous Markov chain, conditioning on $Y_{T_{y_0}}$ yields
\begin{align*}
    \E{T}
    = \E{T_{y_0}} + \E{\E{T-T_{y_0}\mid Y_{T_{y_0}}}}
    &\le \E{T_{y_0}} + \sup_{1\le y\le y_0}\E{T\mid Y_0=y}\\
    &= \Oh{n\log n},
\end{align*}
where we used \Cref{eq:unif:runtime_y0} and \Cref{eq:unif:endgame_exp}.

Moreover, conditioning on $Y_{T_{y_0}}$ and using \Cref{eq:unif:endgame_tail} gives
\begin{equation*}
    \Prob{T-T_{y_0}>n\log n}
    \le \sup_{1\le y\le y_0}\Prob{T>n\log n\mid Y_0=y}
    = \oh{1}.
\end{equation*}
Together with $T_{y_0}=\Oh{n\log n}$ a.a.s., a union bound implies $T=\Oh{n\log n}$ a.a.s.

Case $\chi > \chi_\mathrm{unif}$. Fix $\chi>\chi_\mathrm{unif}$.
By \Cref{cor:unif:small_alpha_sign} and continuity of $\alpha\mapsto D_\chi(\alpha)$, there exists $\alpha^*>0$ such that
$D_\chi(\alpha)<0$ for all $\alpha\in(0,\alpha^*]$.
Choose constants $0<a<b\le \alpha^*$.
Then $\max_{\alpha\in[a,b]}D_\chi(\alpha)<0$, and hence we may assume that
\begin{equation*}
    D_\chi(\alpha)\le -2\eps
    \quad\text{for all }\alpha\in[a,b].
\end{equation*}
Applying \Cref{lem:pre:simplified_drift_conditions} yields that, for all sufficiently large $n$,
the hypotheses of \Cref{thm:pre:simplified_drift} hold for $(Y_t)_{t\ge 0}$ on the interval $[a_n,b_n]$ with
$a_n\coloneqq\lfloor an\rfloor$ and $b_n\coloneqq\lfloor bn\rfloor$. Indeed, condition~(1) holds with $\eps$ as above, and condition~(2) holds with $\delta=1$ and $r(\ell)=e^\chi$, where $\ell\coloneqq b_n - a_n$.
Define the hitting time $T^*\coloneqq \min\SetBuilder{t\ge 0}{Y_t\le a_n}$.
Then \Cref{thm:pre:simplified_drift} implies that there exists a constant $c^*>0$ such that
\begin{equation*}
    \Prob{T^*\le 2^{c^*\ell/r(\ell)}\mid Y_0\ge b_n} = 2^{-\Om{\ell/r(\ell)}} = 2^{-\Om{n}},
\end{equation*}
for all sufficiently large $n$, where we used $r(\ell)=e^\chi\le\Oh{1}$ and $\ell \le (b - a)n\le\Oh{n}$.
Hence with probability $1-2^{-\Om{n}}$ we have $T^*\ge 2^{\Om{n}}$ on the event $\{Y_0\ge b_n\}$.
Since reaching the optimum $Y_t=0$ implies reaching $Y_t\le a_n$, the runtime $T$ satisfies $T\ge T^*$ on $\{Y_0\ge b_n\}$.

Under uniform random initialisation, $Y_0\sim\Bin(n,1/2)$, and for any fixed $b<1/2$ a Chernoff bound gives
$\Prob{Y_0\ge b_n}=1-e^{-\Om{n}}$.
Therefore $T=2^{\Om{n}}$ a.a.s.
Moreover,
\begin{equation*}
    \E{T}\ge \Prob{Y_0\ge b_n}\,\E{T\mid Y_0\ge b_n}
    \ge (1-e^{-\Om{n}})\,2^{\Om{n}}
    = 2^{\Om{n}},
\end{equation*}
so the runtime is exponential in expectation as well.
\end{proof}
\section{Conclusion}
\label{sec:conclusion}

We studied the $\opoEA$ in dynamic linear environments, where a weight vector is redrawn in every generation.
A central empirical observation reported by Vermetten et al.~\cite{vermettenEmpiricalAnalysisDynamic2024} is that for large mutation rates, the algorithm fails to get close to the optimum even under generous evaluation budgets, and instead stalls at a surprisingly large distance from the optimum. Our results locate this distance for Dynamic Binary Value (DBV) and identify the critical mutation rate for the Uniform $(0,1)$ distribution of weights. Moreover, we show that the critical mutation parameter $\chi_\mathrm{dbv}\approx 1.59362$ for DBV is universal in the sense that any mutation parameter $\chi < \chi_\mathrm{dbv}$ guarantees optimisation time $\Oh{n\log n}$ on all permutation-invariant dynamic functions. This fixes a bug in a previous proof and extends and strengthens the results that were previously claimed.

\paragraph{Open problems.} 
We left some questions unanswered in this paper. First of all, we do not locate the plateau by an explicit formula for the Uniform $(0,1)$ distribution. Moreover, since it is difficult to realize DBV in a standard  dynamic framework due to floating point limitations (the DBV weights get forbiddingly large), Vermetten~et~al.\ \cite{vermettenEmpiricalAnalysisDynamic2024} also introduced further weight distributions such as \textsc{PowersOfTwo} and Pareto as approximations of DBV. The performance differences to DBV were empirically found to be small, and it would be interesting if they could be analytically quantified.

Widening the scope to other distributions for dynamic linear functions, it would be interesting to understand which weight distributions lead to a critical mutation parameter for the $\opoEA$ and which don't. For example, we show that the uniform distribution on $(0,1)$ leads to a threshold of $\chi_\mathrm{unif}\approx 2.76531$. On the other hand, the constant distribution with value $1$ just recovers the \onemax function which is known \emph{not} to have a critical mutation rate, since the $\opoEA$ is efficient on \onemax for every constant mutation parameter $\chi >0$~\cite{witt2013tight}. Particularly intriguing is whether there are distributions with compact domain which show a dichotomy. In other words: are there weight distributions with domain $[1,C]$, $C>1$ for which the $\opoEA$ needs exponential time on the corresponding dynamic linear function for large $\chi$? If yes, which is the infimal $C$ for which that can happen? 

Finally, the situation for more complex algorithms is far from understood. While the behaviour close to the optimum has been studied for many algorithms in~\cite{lengler2020dichotomy}, it is known that this is not always the hardest part for optimisation, especially if the population size is larger than one~\cite{lenglerRiedi2022,lenglerZou2021}. Our understanding even for simple algorithms like the $(\mu+1)$ GA is still very limited.

\begin{credits}
\subsubsection{\ackname} This project was supported by the Swiss National Science Foundation [grant number 0003390].
\end{credits}
%
%
%
\bibliographystyle{splncs04}
\bibliography{bib/refs}

\appendix

\section{Appendix}
\label{app:proofs}

\subsection{A mistake in the previous proof}
\label{sec:erratum}

To treat states away from the optimum, the proof of \cite[Theorem~4~a)]{lengler1+1EANoisyLinear2018} introduces an auxiliary search point $z^{(t)}$ whose set of zero-bits is a subset of the set of zero-bits of $x^{(t)}$, and then claims that at each step the difference between $Y_t$ and $Y_t'$ can only decrease or remain the same.
This is false.

Let $\mathcal Z(x)\subseteq[n]$ denote the set of zero-bits of a bit string $x$.
Assume that for some time $t$ we have
\begin{equation*}
\mathcal Z(z^{(t)})\subsetneq \mathcal Z(x^{(t)}).
\end{equation*}
Choose indices
\begin{equation*}
a\in \mathcal Z(x^{(t)})\setminus \mathcal Z(z^{(t)}),
\qquad
b,c\notin \mathcal Z(x^{(t)}),
\end{equation*}
and apply the same mutation set $F=\{a,b,c\}$ to both search points.

For $x^{(t)}$, exactly one zero-bit and two one-bits are flipped, so $K=3$ and $I=1$.
Hence, if the offspring is accepted, the number of zero-bits increases by
\begin{equation*}
K-2I = 1.
\end{equation*}

For $z^{(t)}$, all flipped bits are one-bits, so $I'=0$.
Hence only $1\to 0$ flips occur, and since all weights are strictly positive, the offspring is rejected.

Now consider a step in which the realised weights satisfy
\begin{equation*}
W_a>W_b+W_c.
\end{equation*}
Then the offspring of $x^{(t)}$ is accepted, whereas the offspring of $z^{(t)}$ is rejected.
Consequently,
\begin{equation*}
Y_{t+1}=Y_t+1,
\qquad
Y'_{t+1}=Y'_t,
\end{equation*}
and therefore
\begin{equation*}
Y_{t+1}-Y'_{t+1}=(Y_t-Y'_t)+1.
\end{equation*}
Thus the difference increases, contradicting the claim in \cite{lengler1+1EANoisyLinear2018}.

The importance of this point is that the monotonicity of $Y_t-Y_t'$ is exactly what would imply
\begin{equation*}
Y_t-Y_{t+1} \ge Y_t'-Y_{t+1}'
\end{equation*}
for every coupled step.
This is the step used in \cite{lengler1+1EANoisyLinear2018} to transfer a positive drift estimate from the auxiliary process, started at
\begin{equation*}
\tilde y = n\log^{-1/2} n,
\end{equation*}
to the original process in the whole region
\begin{equation*}
Y_t\in[n\log^{-1/4}n,n].
\end{equation*}

Since this domination is not established, the proof does not show that there exists $\eps>0$ such that
\begin{equation*}
\Delta(y)\ge \eps \log^{-1/2} n
\qquad\text{for all } y\in[n\log^{-1/4}n,n].
\end{equation*}
Consequently, the application of the Additive Drift Theorem to the stopping time $T' \coloneqq \inf\{t\ge 0: Y_t\le n\log^{-1/4}n\}$ is not justified.
Hence the argument no longer proves that the process reaches the sublinear regime
in $\Oh{n\log^{1/2}n}$ steps, neither in expectation nor asymptotically almost surely.

It is important to note that this gap concerns only the argument away from the optimum. The analysis in the near-optimal regime, where $y=\oh{n}$, remains valid,
since there the drift estimate is used directly and no coupling argument is needed.

We do not see an obvious way to salvage the coupling argument from~\cite{lengler1+1EANoisyLinear2018}. Instead, we use a different strategy by using a direct lower bound on the approximate drift $D_\chi(\alpha)$
for all $\alpha\in[0,1]$, and hence, via the approximation result for $\Delta(y)$,
a valid lower bound on the actual drift for all states above a constant threshold.
This yields the required $\Oh{n\log n}$ runtime bound from arbitrary initial states.

\subsection{Proof of \Cref{prop:pre:drift_convergence}}
\label{app:prop:pre:drift_convergence}

\begin{proof}
Fix $\chi\in(0,\infty)$ and $\alpha\in[0,1]$, and set $y\coloneqq y(\alpha)=\lfloor \alpha n\rfloor$.
Condition on the event $Y_t=y$. Recall that $K\sim\Bin(n,\chi/n)$ denotes the number of flipped bits.
We define a pair $(K,I_\Hyp)$ as follows: first draw $K$, and conditional on $K=k$ draw
$I_\Hyp\sim\Hyp(n,y,k)$, i.e.\ $I_\Hyp$ is the number of flipped zero-bits when $Y_t=y$.
Likewise, let $(K,I_\Bin)$ be such that conditional on $K=k$ we have $I_\Bin\sim\Bin(k,\alpha)$.

Moreover, by \Cref{prop:pre:explicit_drift} we have
\begin{equation*}
    \Delta(y)=\E{(2I_\Hyp-K)\ind_A\mid Y_t=y}.
\end{equation*}
Recall from \Cref{eq:pre:conditioned_explicit_drift} that, conditional on $Y_t=y$, we have
\begin{equation*}
    \Delta(y) = \sum_{k = 0}^n\Prob{K=k}\sum_{i = 0}^{\Min{k}{y}} (2i-k)\,p_A(k,i)\,\Prob{I = i\mid(Y_t=y,K=k)}
\end{equation*}
Since $\Prob{I=i\mid (Y_t=y,K=k)}=0$ for all $i>\Min{k}{y}$, we may equivalently write
\begin{equation}
\label{eq:app:delta_double_sum_hyp}
    \Delta(y) = \sum_{k = 0}^n \Prob{K = k}\sum_{i = 0}^k (2i - k)\,p_A(k,i)\,\Prob{I = i\mid (Y_t = y,K=k)}
\end{equation}
We start by approximating $\Hyp(n,y,k)$ by $\Bin(k,\alpha)$. For $k\ge 1$ and $i\in\{0,1,\dots,k\}$ define
\begin{equation*}
    h_k(i)\coloneqq (2i-k)\,p_A(k,i).
\end{equation*}
Then $\abs{h_k(i)/k}\le 1$ for all $k,i$, since $\abs{2i-k}\le k$ and $p_A(k,i)\in[0,1]$.

We first bound the total variation distance between the two laws of $I$ for fixed $k$.
Couple $k$ draws without replacement from $\{1,\dots,n\}$ (yielding $\Hyp(n,y,k)$) with $k$ i.i.d.\ uniform draws from $\{1,\dots,n\}$ (yielding $\Bin(k,y/n)$).
If no collision occurs among the i.i.d.\ draws, then both sampling procedures select the same set of bits, hence the two counts coincide.
By \cite[Proposition~4.7]{levinMarkovChainsMixing} and a union bound over collisions,
\begin{equation}
\label{eq:app:tv_hyp_bin_yn}
    \totVar{\Hyp(n,y,k)-\Bin(k,y/n)}
    \le \Prob{\{\text{a collision occurs}\}}
    \le \binom{k}{2}\frac{1}{n}.
\end{equation}
Next, we couple $\Bin(k,y/n)$ and $\Bin(k,\alpha)$ by coupling the $k$ underlying Bernoulli trials one-by-one.
Since $\totVar{\Ber(p)-\Ber(q)}=\abs{p-q}$ and total variation satisfies the triangle inequality (\cite[Remark~4.4]{levinMarkovChainsMixing}), this yields the standard product bound
\begin{equation}
\label{eq:app:tv_bin_bin_alpha}
    \totVar{\Bin(k,y/n)-\Bin(k,\alpha)}
    \le k\,\abs{y/n-\alpha}
    \le \frac{k}{n},
\end{equation}
where we used $\abs{\lfloor \alpha n\rfloor/n-\alpha}\le 1/n$.
Combining \Cref{eq:app:tv_hyp_bin_yn} and \Cref{eq:app:tv_bin_bin_alpha} via the triangle inequality gives, for every $1\le k\le n$,
\begin{equation}
\label{eq:app:tv_hyp_bin_alpha}
    \totVar{\Hyp(n,y,k)-\Bin(k,\alpha)}
    \le \binom{k}{2}\frac{1}{n}+\frac{k}{n}.
\end{equation}

We now bound the effect of this approximation on the inner sum.
By \cite[Proposition~4.5]{levinMarkovChainsMixing} and $\abs{h_k(i)/k}\le 1$,
\begin{equation*}
    \abs*{\E*{\frac{h_k(I_\Hyp)}{k}}-\E*{\frac{h_k(I_\Bin)}{k}}}\le 2\,\totVar{\Hyp(n,y,k)-\Bin(k,\alpha)}.
\end{equation*}
Multiplying by $k$ and using \Cref{eq:app:tv_hyp_bin_alpha} yields, for $1\le k\le n$,
\begin{equation}
\label{eq:app:inner_error}
    \abs{\E{h_k(I_\Hyp)}-\E{h_k(I_\Bin)}}
    \le 2k\Bigl(\binom{k}{2}\frac{1}{n}+\frac{k}{n}\Bigr)
    = \frac{k^3+k^2}{n}.
\end{equation}
For $k=0$ both sides equal $0$, so \Cref{eq:app:inner_error} holds for all $0\le k\le n$.

Define the ``intermediate'' drift $D_\chi^{(n)}$ where only the law of $I$ is approximated by a binomial:
\begin{equation}
\label{eq:app:drift_binI_def}
    D^{(n)}_\chi(\alpha)
    \coloneqq \sum_{k=0}^{n}\Prob{K=k}\sum_{i=0}^{k}(2i-k)\,p_A(k,i)\,\Prob{\Bin(k,\alpha)=i}.
\end{equation}

With this coupling, the inner sum in \Cref{eq:app:delta_double_sum_hyp} equals $\E{h_k(I_\Hyp)\mid K=k}$.
Hence, by the law of total expectation,
\begin{equation*}
    \Delta(y)=\E{h_K(I_\Hyp)}.
\end{equation*}
Similarly, by \Cref{eq:app:drift_binI_def} and the definition of $(K,I_\Bin)$ we have
\begin{equation*}
    D^{(n)}_\chi(\alpha)=\E{h_K(I_\Bin)}.
\end{equation*}
Averaging \Cref{eq:app:inner_error} over $K\sim\Bin(n,\chi/n)$ gives
\begin{equation}
\label{eq:app:error_replace_I}
    \abs{\Delta(y)-D^{(n)}_\chi(\alpha)}
    \le \frac{1}{n}\E{K^3+K^2}.
\end{equation}
Using falling factorials $(x)_r\coloneqq x(x-1)\cdots(x-r+1)$, we have the identities
\begin{equation*}
    K^2=(K)_2+K
    \quad\text{and}\quad
    K^3=(K)_3+3(K)_2+K.
\end{equation*}
Hence $\E{K^3+K^2}=\E{(K)_3}+4\E{(K)_2}+2\E{K}$.
Moreover,
\begin{equation*}
    \E{(K)_3}=(n)_3\Bigl(\frac{\chi}{n}\Bigr)^3 \le \chi^3,\quad
    \E{(K)_2}=(n)_2\Bigl(\frac{\chi}{n}\Bigr)^2 \le \chi^2,\quad
    \E{K}=\chi.
\end{equation*}
Inserting into \Cref{eq:app:error_replace_I} yields
\begin{equation}
\label{eq:app:error_replace_I_final}
    \abs{\Delta(y)-D^{(n)}_\chi(\alpha)}
    \le \frac{\chi^3+4\chi^2+2\chi}{n}.
\end{equation}

We now approximate $K\sim\Bin(n,\chi/n)$ by $\Poi(\chi)$. For $k\ge 1$ define
\begin{equation*}
    g_\alpha(k)\coloneqq \frac{1}{k}\sum_{i=0}^{k}(2i-k)\,p_A(k,i)\,\Prob{\Bin(k,\alpha)=i},
    \quad g_\alpha(0)\coloneqq 0.
\end{equation*}
Then $\abs{g_\alpha(k)}\le 1$ for all $k$, since $\abs{2i-k}\le k$ and $p_A(k,i)\le 1$.
By construction,
\begin{equation*}
    D^{(n)}_\chi(\alpha)=\E{K\,g_\alpha(K)} \quad\text{for } K\sim\Bin(n,\chi/n),
\end{equation*}
and
\begin{equation*}
    D_\chi(\alpha)=\E{K\,g_\alpha(K)} \quad\text{for } K\sim\Poi(\chi),
\end{equation*}
where $D_\chi(\alpha)=D(\chi,\alpha)$ is defined in \Cref{eq:pre:approximate_drift}.

Let $K_n\sim\Bin(n,\chi/n)$ and $K_\Poi\sim\Poi(\chi)$. Since $\E{K_n}=\E{K_\Poi}=\chi>0$, we may size-bias both distributions.
By \cite[Equation~(9)]{arratiaSizeBiasOne2018}, for any bounded measurable $g$ and any non-negative $K$ with mean $\chi>0$,
\begin{equation*}
    \E{K\,g(K)}=\chi\,\E{g(K^*)}.
\end{equation*}
Moreover, by \cite[Equation~(30)]{arratiaSizeBiasOne2018} we have the explicit size-biased laws
\begin{equation*}
    K_n^*\overset{d}{=}1+\Bin(n-1,\chi/n)
    \quad\text{and}\quad
    K_\Poi^*\overset{d}{=}1+\Poi(\chi).
\end{equation*}
Applying these identities with $g=g_\alpha$ gives
\begin{align*}
    \abs{D^{(n)}_\chi(\alpha)-D_\chi(\alpha)}
    &= \abs{\E{K_n\,g_\alpha(K_n)}-\E{K_\Poi\,g_\alpha(K_\Poi)}} \\
    &= \chi\,\abs{\E{g_\alpha(1+K_{n-1})}-\E{g_\alpha(1+K_\Poi)}},
\end{align*}
where $K_{n-1}\sim\Bin(n-1,\chi/n)$.
By \cite[Proposition~4.5]{levinMarkovChainsMixing} and $\abs{g_\alpha(k)}\le 1$ we obtain
\begin{equation}
\label{eq:app:error_replace_K_tv}
    \abs{D^{(n)}_\chi(\alpha)-D_\chi(\alpha)}
    \le 2\chi\,\totVar{\Bin(n-1,\chi/n)-\Poi(\chi)}.
\end{equation}

To bound this total variation distance, we use the triangle inequality (\cite[Remark~4.4]{levinMarkovChainsMixing}).
First, couple $\Bin(n-1,\chi/n)$ and $\Bin(n,\chi/n)$ by writing $K_n=K_{n-1}+X_n$ with $X_n\sim\Ber(\chi/n)$ and independent.
Then \cite[Proposition~4.7]{levinMarkovChainsMixing} yields
\begin{equation*}
    \totVar{\Bin(n-1,\chi/n)-\Bin(n,\chi/n)}
    \le \Prob{K_{n-1}\neq K_n}
    = \Prob{X_n=1}
    = \frac{\chi}{n}.
\end{equation*}
Second, by the Poisson approximation bound in \cite[Theorem~4.6]{rossFundamentalsSteinsMethod2011},
\begin{equation*}
    \totVar{\Bin(n,\chi/n)-\Poi(\chi)} \le \frac{\chi^2}{n}.
\end{equation*}
Combining, we get
\begin{equation*}
    \totVar{\Bin(n-1,\chi/n)-\Poi(\chi)} \le \frac{\chi}{n}+\frac{\chi^2}{n}.
\end{equation*}
Inserting into \Cref{eq:app:error_replace_K_tv} yields
\begin{equation}
\label{eq:app:error_replace_K_final}
    \abs{D^{(n)}_\chi(\alpha)-D_\chi(\alpha)}
    \le \frac{2\chi^3+2\chi^2}{n}.
\end{equation}

Finally, we combine the two approximation errors. By the triangle inequality,
\begin{equation*}
    \abs{\Delta(y)-D_\chi(\alpha)}
    \le \abs{\Delta(y)-D^{(n)}_\chi(\alpha)} + \abs{D^{(n)}_\chi(\alpha)-D_\chi(\alpha)}.
\end{equation*}
Combining \Cref{eq:app:error_replace_I_final} and \Cref{eq:app:error_replace_K_final} gives
\begin{equation*}
    \abs{\Delta(\lfloor \alpha n\rfloor)-D_\chi(\alpha)}
    \le \frac{3\chi^3+6\chi^2+2\chi}{n}.
\end{equation*}
The bound is uniform in $\alpha\in[0,1]$ since the estimates above depend on $\alpha$ only through
$\abs{\lfloor \alpha n\rfloor/n-\alpha}\le 1/n$.
In particular, for every fixed $\alpha\in[0,1]$ we have $\Delta(\lfloor \alpha n\rfloor)\to D_\chi(\alpha)$ as $n\to\infty$.
\end{proof}

\subsection{Proof of \Cref{prop:pre:right_derivative}}
\label{app:prop:pre:right_derivative}

\begin{proof}
Recall that $D_\chi(\alpha)=D(\chi,\alpha)$ is defined by conditioning on $K\sim\Poi(\chi)$ and
$I\mid(K=k)\sim\Bin(k,\alpha)$, hence
\begin{equation}
\label{eq:app:Dchi_double_sum}
    D_\chi(\alpha)
    = \sum_{k=0}^{\infty} e^{-\chi}\frac{\chi^k}{k!}\sum_{i=0}^{k} (2i-k)\,p_A(k,i)\binom{k}{i}\alpha^i(1-\alpha)^{k-i}.
\end{equation}
For fixed $k$, let $I\sim\Bin(k,\alpha)$ and write the inner sum as
\begin{equation*}
    S_k(\alpha)\coloneqq \sum_{i=0}^{k} (2i-k)\,p_A(k,i)\binom{k}{i}\alpha^i(1-\alpha)^{k-i}
    = \E{(2I-k)\,p_A(k,I)}.
\end{equation*}

We expand $S_k(\alpha)$ at $\alpha=0$ and bound the remainder uniformly in $\alpha\in[0,1]$.
If $I=0$, then all flipped bits are one-bits, so the offspring cannot be accepted because all weights are non-negative and hence
$p_A(k,0)=0$ for all $k\ge 1$. Moreover, for $k=0$ the factor $2i-k$ is $0$. Thus the $i=0$ term never contributes.

For the $i=1$ term we have
\begin{equation*}
    (2-k)\,p_A(k,1)\binom{k}{1}\alpha(1-\alpha)^{k-1}
    = k(2-k)\,p_A(k,1)\,\alpha + R_{k,1}(\alpha),
\end{equation*}
where
\begin{equation*}
    R_{k,1}(\alpha)
    = k(2-k)\,p_A(k,1)\,\alpha\bigl((1-\alpha)^{k-1}-1\bigr).
\end{equation*}
Since $p_A(k,1)\in[0,1]$ and $1-(1-\alpha)^{k-1}\le (k-1)\alpha$ for $\alpha\in[0,1]$, we obtain
\begin{equation}
\label{eq:app:i_eq_one_bound}
    \abs{R_{k,1}(\alpha)}
    \le k\abs{2-k}\,\alpha\bigl(1-(1-\alpha)^{k-1}\bigr)
    \le k^2\alpha\,(k-1)\alpha
    \le k^3\alpha^2.
\end{equation}

For the contribution of indices $i\ge 2$, using $\abs{2i-k}\le k$ and $p_A(k,i)\in[0,1]$ yields
\begin{equation*}
    \abs*{\sum_{i=2}^{k} (2i-k)\,p_A(k,i)\binom{k}{i}\alpha^i(1-\alpha)^{k-i}}
    \le k\,\Prob{I\ge 2}.
\end{equation*}
To bound this tail probability, note that $I\ge 2$ implies that there exists a pair of distinct trials that both succeed.
By a union bound over the $\binom{k}{2}$ pairs,
\begin{equation}
\label{eq:app:i_ge_two_bound}
    \Prob{\Bin(k,\alpha)\ge 2}
    \le \binom{k}{2}\alpha^2
    \le \frac{k^2}{2}\alpha^2,
\end{equation}
and therefore the $i\ge 2$ contribution is bounded (in absolute value) by $\frac{1}{2}k^3\alpha^2$.

Combining the \Cref{eq:app:i_eq_one_bound} and \Cref{eq:app:i_ge_two_bound} we conclude that for all $k\ge 0$ and all $\alpha\in[0,1]$,
\begin{equation}
\label{eq:app:inner_expansion_fixed_k_new}
    S_k(\alpha)
    = k(2-k)\,p_A(k,1)\,\alpha + R_k(\alpha),
    \quad
    \abs{R_k(\alpha)}\le \frac{3}{2}\,k^3\alpha^2.
\end{equation}

Insert \Cref{eq:app:inner_expansion_fixed_k_new} into \Cref{eq:app:Dchi_double_sum}. Writing $K\sim\Poi(\chi)$ and observing that the term for $k=i=0$ vanishes due to the factor $2i-k=0$, we obtain
\begin{equation*}
    D_\chi(\alpha)
    = \alpha\sum_{k=1}^{\infty} e^{-\chi}\frac{\chi^k}{k!}\,k(2-k)\,p_A(k,1)
      +\sum_{k=1}^{\infty} e^{-\chi}\frac{\chi^k}{k!}\,R_k(\alpha).
\end{equation*}
Using $\abs{R_k(\alpha)}\le \frac{3}{2}k^3\alpha^2$ and the finiteness of the third moment of a Poisson random variable,
\begin{equation*}
    \abs*{\sum_{k=0}^{\infty} e^{-\chi}\frac{\chi^k}{k!}\,R_k(\alpha)}
    \le \frac{3}{2}\alpha^2 \sum_{k=0}^{\infty} e^{-\chi}\frac{\chi^k}{k!}\,k^3
    = \frac{3}{2}\alpha^2\,\E{K^3}
    = \Oh{\alpha^2},
\end{equation*}
where the hidden constant depends only on $\chi$ via $\E{K^3}$.
Hence
\begin{equation}
\label{eq:app:Dchi_expansion}
    D_\chi(\alpha)
    = \alpha\sum_{k=1}^{\infty} e^{-\chi}\frac{\chi^k}{k!}\,k(2-k)\,p_A(k,1) + \Oh{\alpha^2}.
\end{equation}

Finally, $D_\chi(0)=0$, so dividing \Cref{eq:app:Dchi_expansion} by $\alpha$ yields
\begin{equation*}
    \frac{D_\chi(\alpha)-D_\chi(0)}{\alpha}
    = \sum_{k=1}^{\infty} e^{-\chi}\frac{\chi^k}{k!}\,k(2-k)\,p_A(k,1) + \Oh{\alpha}.
\end{equation*}
Letting $\alpha\to 0^+$ gives the claimed formula for $\partial_+ D_\chi(0)$.
\end{proof}

\subsection{Proof of \Cref{lem:pre:constant_drift}}
\label{app:lem:pre:constant_drift}

\begin{proof}
Write $p\coloneqq\chi/n$. Fix $n$ sufficiently large such that $n\ge n_0$ and $p\le 1/2$. In particular,
\begin{equation}
\label{eq:endgame:qn_lower}
    (1-p)^{n-1}\ge e^{-2p(n-1)}\ge e^{-2\chi}.
\end{equation}
We prove (1) and (2) for such $n$.

For $y\ge 1$, let $E_y$ be the event that mutation flips exactly one bit and this bit is one of the $y$ zero-bits.
If $Y_t=y$, then
\begin{equation}
\label{eq:endgame:PyEy}
    \Prob{E_y\mid Y_t=y}=y\,p(1-p)^{n-1}.
\end{equation}
On $E_y$ the offspring flips one $0\to 1$ bit and no other bit. Since weights are assumed to be positive almost surely, this strictly increases fitness,
so the offspring is accepted and $Y_{t+1}=y-1$.

Consequently, starting from a state $y\le y_0$, if in each accepted generation $t$ until reaching $0$ the corresponding event $E_{Y_t}$ occurs, then $Y_t$ decreases by $1$ each time and the optimum is reached after at most $y$ accepted generations. We refer to this outcome as a lucky strike.

Let $A$ denote the acceptance event and let $I$ be the number of flipped zero-bits.
If $I=0$, then only $1\to 0$ flips occur fitness can only decrease, hence acceptance is impossible.
Therefore $A\subseteq\{I\ge 1\}$ and, for $Y_t=y$,
\begin{equation}
\label{eq:endgame:acceptance_upper_bound}
    \Prob{A\mid Y_t=y}\le \Prob{I\ge 1\mid Y_t=y}\le \E{I\mid Y_t=y}=y\,p.
\end{equation}
Combining \Cref{eq:endgame:PyEy} and \Cref{eq:endgame:acceptance_upper_bound} gives, for all $y\ge 1$,
\begin{equation}
\label{eq:endgame:Ey_given_accept}
    \Prob{E_y\mid A,\,Y_t=y}
    \ge \frac{y\,p(1-p)^{n-1}}{y\,p}
    =(1-p)^{n-1}.
\end{equation}
Now fix any $y\in\{1,\dots,y_0\}$.
Consider the successive accepted generations starting from a generation $t$ when $Y_t=y$, and stop if either the process reaches $0$ or we encounter the first accepted generation $t'$ in which $E_{Y_{t'}}$ does not occur.
At each accepted generation $t$ before reaching $0$, the current value of $Y_t$ lies in $\{1,\dots,y_0\}$, and by
\Cref{eq:endgame:Ey_given_accept} the conditional probability that the accepted generation is of type $E_{Y_t}$ is at least
$(1-p)^{n-1}$. Iterating this bound over at most $y_0$ accepted generations yields that the probability of a lucky strike is at least
\begin{equation}
\label{eq:endgame:lucky_strike_prob}
    \Prob{\{\text{lucky strike}\}\mid Y_t=y}\ge \bigl((1-p)^{n-1}\bigr)^{y_0}\ge e^{-2\chi y_0},
\end{equation}
where the last inequality follows from \Cref{eq:endgame:qn_lower}.

Assume $Y_0\in\{1,\dots,y_0\}$. We proceed as follows.
Whenever the process is in $\{1,\dots,y_0\}$, we start an attempt and watch accepted generations until either a lucky strike occurs
(success) or the first accepted generation $t$ occurs in which $E_{Y_t}$ does not occur (failure).
If an attempt fails and the new state is still in $\{1,\dots,y_0\}$, we start a new attempt immediately.
If instead the attempt fails and the new state is $>y_0$, we wait until the first time the process returns to $\{1,\dots,y_0\}$ and then
start the next attempt there.

By \Cref{eq:endgame:lucky_strike_prob}, each time we start an attempt from a state in $\{1,\dots,y_0\}$, the conditional probability of
success is at least $e^{-2\chi y_0}=:p_{\mathrm{ls}}>0$.
Let $N$ be the number of attempts until the first success. Then for all $m\ge 0$,
\begin{equation*}
    \Prob{N>m}\le (1-p_{\mathrm{ls}})^m,
\end{equation*}
since each attempt succeeds with conditional probability at least $p_{\mathrm{ls}}$, and hence
\begin{equation}
\label{eq:endgame:EN}
    \E{N}=\sum_{m\ge 0}\Prob{N>m}\le \frac{1}{p_{\mathrm{ls}}}\le\Oh{1}.
\end{equation}

For $Y_t\ge 1$, an accepted generation occurs whenever $E_{Y_t}$ occurs, hence
\begin{equation*}
    \Prob{A\mid Y_t\ge 1}\ge \Prob{E_{Y_t}\mid Y_t\ge 1}\ge p(1-p)^{n-1}.
\end{equation*}
Therefore the expected waiting time (in generations) until the next accepted generation is at most
\begin{equation*}
    \frac{1}{p(1-p)^{n-1}}=\frac{n}{\chi}\,\frac{1}{(1-p)^{n-1}}\le\Oh{n},
\end{equation*}
using \Cref{eq:endgame:qn_lower}. Since an attempt involves at most $y_0$ accepted generations before it either succeeds or fails, the
expected number of generations spent in $\{1,\dots,y_0\}$ during an attempt is $\Oh{n}$.

If an attempt fails in generation $t$ by leaving $\{1,\dots,y_0\}$, let $Z$ be the value of $Y_{t + 1}$ (i.e. immediately after an attempt fails).
Let $K\sim\Bin(n,p)$ be the mutation size in that generation, so $\E{K}=\chi$ and
$\E{K^2}=\chi(1-p)+\chi^2\le \chi+\chi^2$.
Fix $y\in\{1,\dots,y_0\}$. Conditional on $K=k$, acceptance implies that at least one of the $k$ flipped positions is a zero-bit, hence
\begin{equation*}
    \Prob{A\mid K=k,\,Y_t=y}\le \Prob{I\ge 1\mid K=k,\,Y_t=y}\le k\,\frac{y}{n}.
\end{equation*}
Thus,
\begin{equation*}
    \E{K\ind_A\mid Y_t=y}\le \sum_{k=0}^n k\,\Prob{K=k}\,k\,\frac{y}{n}=\frac{y}{n}\E{K^2}.
\end{equation*}
On the other hand, $E_y\subseteq A$ implies
$\Prob{A\mid Y_t=y}\ge \Prob{E_y\mid Y_t=y}=\frac{y}{n}\chi(1-p)^{n-1}$.
Using $\E{K\mid A,\,Y_t=y}=\E{K\ind_A\mid Y_t=y}/\Prob{A\mid Y_t=y}$, we obtain
\begin{equation}
\label{eq:endgame:EK_given_accept}
    \E{K\mid A,\,Y_t=y}
    \le \frac{\E{K^2}}{\chi(1-p)^{n-1}}
    \le \frac{\chi+\chi^2}{\chi(1-p)^{n-1}}
    \le\Oh{1},
\end{equation}
where the last bound uses \Cref{eq:endgame:qn_lower}.
At generation $t$, where we fail an attempt, we must have $Y_t\le y_0$ and the state can increase by at most $K$ in one step, hence $Z\le y_0+K$ and
therefore $\E{Z}\le\Oh{1}$.

Now apply the multiplicative drift theorem \cite[Theorem~3]{doerrMultiplicativeDriftAnalysis2012} under the assumption $\Delta(y)\ge (c/n)y$ for $y\ge y_0$.
For all $z\ge y_0$ this yields an expected return time to $\{0,1,\dots,y_0\}$ of at most
\begin{equation*}
    \frac{n}{c}\Bigl(1+\ln(z/y_0)\Bigr).
\end{equation*}
Conditioning on $Z$ and using Jensen's inequality (since $\ln$ is concave) and $\E{Z}\le\Oh{1}$ we obtain a return time of
\begin{equation*}
    \frac{n}{c}\Bigl(1+\E{\ln(Z/y_0)}\Bigr)
    \le \frac{n}{c}\Bigl(1+\ln(\E{Z}/y_0)\Bigr)
    \le\Oh{n},
\end{equation*}
in expectation and a.a.s.
Hence the expected number of generations per attempt (including a possible excursion above $y_0$ and the return) is $\Oh{n}$.

In summary, starting from any $y\in\{1,\dots,y_0\}$, the runtime is the sum of the durations of $N$ attempts, hence by \Cref{eq:endgame:EN},
\begin{equation*}
    \sup_{1\le y\le y_0}\E{T\mid Y_0=y}\le\Oh{n},
\end{equation*}
which proves (1).
For (2), Markov's inequality yields for fixed $y\in\{1,\dots,y_0\}$,
\begin{equation*}
    \Prob{T>n\log n\mid Y_0=y}\le \frac{\E{T\mid Y_0=y}}{n\log n}\le\Oh*{\frac{1}{\log n}}=\oh{1},
\end{equation*}
which concludes the proof.
\end{proof}

\subsection{Proof of \Cref{prop:dbv:drift_closed_form}}
\label{sec:app:dbv:drift_closed_form}

\begin{proof}
Recall that for DBV we have $p_A(k,i)=i/k$ for all $k\ge 1$ and $0\le i\le k$.
By definition of the continuous drift approximation (see \Cref{eq:pre:approximate_drift}),
\begin{equation*}
    D(\chi,\alpha)
    =\sum_{k=0}^\infty e^{-\chi}\frac{\chi^k}{k!}
      \sum_{i=0}^k (2i-k)\,p_A(k,i)\,\Prob{I=i\mid K=k},
\end{equation*}
where $K\sim\Poi(\chi)$ and $I\mid(K=k)\sim\Bin(k,\alpha)$.
For $k=0$, recall that we set $p_A(0,0)=0$ hence the $k=0$ term vanishes. This corresponds to the case the parent and offspring are identical, in which case the drift is trivially zero.

For $k\ge 1$, using $p_A(k,i)=i/k$ we obtain
\begin{equation*}
    \sum_{i=0}^k (2i-k)\,p_A(k,i)\,\Prob{I=i\mid K=k}
    =\E{(2I-k)I/k\mid K=k}.
\end{equation*}
For $I\sim\Bin(k,\alpha)$ we have $\E{I}=k\alpha$ and
\begin{equation*}
    \E{I^2}=k\alpha(1-\alpha)+k^2\alpha^2.
\end{equation*}
Thus, for every $k\ge 1$,
\begin{align*}
    \E{(2I-k)I/k\mid K=k}
    &=\frac{1}{k}\Bigl(2\E{I^2\mid K=k}-k\E{I\mid K=k}\Bigr) \\
    &=\frac{1}{k}\Bigl(2\bigl(k\alpha(1-\alpha)+k^2\alpha^2\bigr)-k^2\alpha\Bigr) \\
    &=2\alpha(1-\alpha)+k(2\alpha^2-\alpha).
\end{align*}

Summing over $k\ge 1$ yields
\begin{align*}
    D(\chi,\alpha)
    &=2\alpha(1-\alpha)\sum_{k=1}^\infty e^{-\chi}\frac{\chi^k}{k!}
      +(2\alpha^2-\alpha)\sum_{k=1}^\infty k\,e^{-\chi}\frac{\chi^k}{k!} \\
    &=2\alpha(1-\alpha)\bigl(1-e^{-\chi}\bigr)+(2\alpha^2-\alpha)\chi,
\end{align*}
since $\sum_{k\ge 1} e^{-\chi}\chi^k/k!=\Prob{\Poi(\chi)\ge 1}=1-e^{-\chi}$ and
$\sum_{k\ge 1} k\,e^{-\chi}\chi^k/k!=\E{\Poi(\chi)}=\chi$.

Finally, expand and regroup:
\begin{align*}
    D(\chi,\alpha)
    &=2\alpha(1-e^{-\chi})-2\alpha^2(1-e^{-\chi})+2\chi\alpha^2-\chi\alpha \\
    &=\alpha\bigl(2-\chi-2e^{-\chi}\bigr)-\alpha^2\bigl(2-2\chi-2e^{-\chi}\bigr) \\
    &=\alpha\Bigl(\bigl(2-\chi-2e^{-\chi}\bigr)-\alpha\bigl(2-2\chi-2e^{-\chi}\bigr)\Bigr),
\end{align*}
which is \Cref{eq:dbv:drift_closed_form}.
\end{proof}

\subsection{Proof of \Cref{prop:dbv:discrete_drift_clean}}
\label{sec:app:dbv:discrete_drift_clean}

\begin{proof}
Set $p\coloneqq \chi/n$ and let $K_n\sim\Bin(n,p)$ be the number of flipped bits in one mutation step.
Recall $\alpha\coloneqq y/n$ and $r_n(\chi)\coloneqq \Prob{K_n=0}=(1-p)^n=(1-\chi/n)^n$.

If $K_n=0$, then the offspring equals the parent and the drift is trivially zero.

Assume now that $K_n=k\ge 1$.
Conditional on $Y_t=y$ and $K_n=k$, the number $I$ of flipped zero-bits satisfies $I\sim\Hyp(n,y,k)$.
For DBV, conditional on $(K_n=k,I=i)$, the offspring is accepted with probability $i/k$ (\Cref{prop:dbv:acceptance}).
Hence
\begin{equation*}
    \E{Y_t-Y_{t+1}\mid Y_t=y,\,K_n=k}
    = \E{(2I-k)\,I/k}
    = \frac{1}{k}\Bigl(2\E{I^2}-k\E{I}\Bigr).
\end{equation*}
For $I\sim\Hyp(n,y,k)$ we have $\E{I}=k\alpha$ and
\begin{equation*}
    \Var{I}=k\alpha(1-\alpha)\frac{n-k}{n-1},
    \quad
    \E{I^2}=\Var{I}+(\E{I})^2.
\end{equation*}
Substituting yields, for $k\ge 1$,
\begin{equation*}
    \E{Y_t-Y_{t+1}\mid Y_t=y,\,K_n=k}
    = \frac{2}{n-1}\,\alpha(1-\alpha)(n-k)
    + k\,\alpha(2\alpha-1).
\end{equation*}

Taking expectations over $K_n$ and using $\E{K_n}=\chi$, we obtain
\begin{equation}
\label{eq:dbv:disc_intermediate}
    \Delta(y)
    = \frac{2}{n-1}\,\alpha(1-\alpha)\,\E{(n-K_n)\ind_{\{K_n\ge 1\}}}
    + \alpha(2\alpha-1)\,\E{K_n}.
\end{equation}
Finally,
\begin{equation*}
    \E{(n-K_n)\ind_{\{K_n\ge 1\}}}
    = n\Prob{K_n\ge 1}-\E{K_n}
    = n(1-r_n(\chi))-\chi,
\end{equation*}
so \Cref{eq:dbv:disc_intermediate} becomes
\begin{equation}
\label{eq:dbv:disc_almost}
    \Delta(y)
    = \frac{2}{n-1}\,\alpha(1-\alpha)\bigl(n(1-r_n(\chi))-\chi\bigr)
    + \chi\,\alpha(2\alpha-1).
\end{equation}

It remains to rewrite \Cref{eq:dbv:disc_almost} into the factorised form \Cref{eq:dbv:discrete_drift_clean}.
Expanding \Cref{eq:dbv:disc_almost} and regrouping the coefficients of $\alpha$ and $\alpha^2$ gives
\begin{equation*}
    \Delta(y)
    = \alpha\Bigl(\frac{2\bigl(n(1-r_n(\chi))-\chi\bigr)}{n-1}-\chi\Bigr)
    + \alpha^2\Bigl(2\chi-\frac{2\bigl(n(1-r_n(\chi))-\chi\bigr)}{n-1}\Bigr).
\end{equation*}
After rewriting and factoring out $\frac{n}{n-1}\alpha$ we get 
\begin{equation*}
    \Delta(y)
    = \frac{n}{n-1}\,\alpha\Bigl(\bigl(2-\chi-2r_n(\chi)-\tfrac{\chi}{n}\bigr)
    -\alpha\bigl(2-2\chi-2r_n(\chi)\bigr)\Bigr),
\end{equation*}
which is exactly \Cref{eq:dbv:discrete_drift_clean}.
\end{proof}

\subsection{Proof of \Cref{prop:unif:irwin_hall}}
\label{app:prop:unif:irwin_hall}

\begin{proof}
Fix integers $k\ge 0$ and $0\le i\le k$ and condition on the event $(K=k,I=i)$.
In the selection step, the offspring is accepted iff the total weight gained from the $i$ flipped zero-bits exceeds the total weight lost from the $k-i$ flipped one-bits.
Since the weights are i.i.d.\ $\Unif(0,1)$, we may represent the $k$ flipped weights by i.i.d.\ random variables
$U_1,\dots,U_k\sim\Unif(0,1)$ such that
\begin{equation*}
    p_A(k,i)
    = \Prob{ \sum_{j=1}^{i} U_j\ge\sum_{j=i+1}^{k} U_j }.
\end{equation*}

For $1\le j\le k-i$ define $V_j\coloneqq 1-U_{i+j}$.
Then $V_1,\dots,V_{k-i}$ are i.i.d.\ $\Unif(0,1)$ and independent of $U_1,\dots,U_i$.
Substituting $U_{i+j}=1-V_j$ yields
\begin{equation*}
    p_A(k,i)
    = \Prob{ \sum_{j=1}^{i} U_j + \sum_{j=1}^{k-i} V_j \ge k-i }.
\end{equation*}
The $k$ summands on the left-hand side are i.i.d.\ $\Unif(0,1)$, hence their sum has the Irwin--Hall distribution with parameter~$k$.
Let $S_k\coloneqq \sum_{j=1}^k W_j$ with $W_1,\dots,W_k\overset{\mathrm{i.i.d.}}{\sim}\Unif(0,1)$, and let $F_k(x)\coloneqq \Prob{S_k\le x}$.
Then
\begin{equation*}
    p_A(k,i)=\Prob{S_k\ge k-i}=1-F_k(k-i).
\end{equation*}
By symmetry $S_k\overset{d}{=}k-S_k$ and continuity of the distribution, we have $F_k(x)=1-F_k(k-x)$ for all $x\in[0,k]$
(also noted in \Cref{sec:preliminaries}).
Thus $1-F_k(k-i)=F_k(i)$ and therefore $p_A(k,i)=F_k(i)$.
Evaluating the Irwin--Hall cdf at the integer $x=i$ yields the claimed closed form.
\end{proof}

\subsection{Proof of \Cref{lem:unif:right_derivative}}
\label{app:lem:unif:right_derivative}

\begin{proof}
By \Cref{prop:pre:right_derivative} we have, for all $\chi>0$,
\begin{equation*}
    \partial_+ D_\chi(0)
    = \sum_{k=0}^\infty e^{-\chi}\frac{\chi^k}{k!}\,k(2-k)\,p_A(k,1).
\end{equation*}
In the uniform weights model, \Cref{prop:unif:irwin_hall} yields $p_A(k,1)=1/k!$ for all $k\ge 1$
(and the term $k=0$ vanishes due to the factor $k(2-k)$). Hence
\begin{equation}
\label{eq:unif:g_series}
    g(\chi)
    = \sum_{k=0}^\infty e^{-\chi}\frac{\chi^k}{(k!)^2}\,k(2-k).
\end{equation}

We use modified Bessel functions of the first kind,
\begin{equation*}
    I_a(x)\coloneqq \sum_{m=0}^\infty \frac{1}{m!\,\Gamma(m+a+1)}\Bigl(\frac{x}{2}\Bigr)^{2m+a},
    \quad a\in\Z.
\end{equation*}
Set $x\coloneqq 2\sqrt{\chi}$. One checks (by comparing coefficients) that
\begin{equation}
\label{eq:unif:bessel_identity}
    \sum_{k=0}^\infty \frac{k(2-k)}{(k!)^2}\chi^k
    = xI_1(x) - \frac{x^2}{4}I_0(x).
\end{equation}
Since $I_0(x)>0$ for all $x>0$, we can factor
\begin{equation}
\label{eq:unif:root_equivalence}
    xI_1(x) - \frac{x^2}{4}I_0(x)
    = \frac{x^2}{4}I_0(x)\Bigl(\frac{4I_1(x)}{xI_0(x)}-1\Bigr).
\end{equation}
Consequently, for $\chi>0$ we have $g(\chi)=0$ if and only if
\begin{equation}\label{eq:unif:root_equation}
    \frac{4I_1(x)}{xI_0(x)}=1,
    \quad x=2\sqrt{\chi},
\end{equation}
since the factor $\frac{x^2}{4}I_0(x)$ in \Cref{eq:unif:root_equivalence} is strictly positive.

We now address the root and its uniqueness. Let $v_0(x)\coloneqq x\,\frac{I_0(x)}{I_1(x)}$ as in \cite[Equation~3.71]{bariczGeneralizedBesselFunctions2010}.
By \cite[Theorem~3.24]{bariczGeneralizedBesselFunctions2010}, the function $v_0\in(0,\infty)$ is strictly increasing on $(0,\infty)$.
Therefore the function
\begin{equation*}
    h(x)\coloneqq \frac{4I_1(x)}{xI_0(x)}=\frac{4}{v_0(x)}
\end{equation*}
is strictly decreasing on $(0,\infty)$.

From the series expansions of $I_0$ and $I_1$ we have
\begin{equation*}
    I_0(x)=1+\Oh{x^2}
    \quad\text{and}\quad
    I_1(x)=\frac{x}{2}+\Oh{x^3},
\end{equation*}
and thus
\begin{equation*}
    h(x)=\frac{4I_1(x)}{xI_0(x)} = 2+\Oh{x^2}>1,
\end{equation*}
for sufficiently small $x>0$.
On the other hand, by comparing the integral representations of $I_0$ and $I_1$ \cite[\href{https://dlmf.nist.gov/10.32.E3}{(10.32.3)}]{NIST:DLMF} one obtains
\begin{equation*}
    \frac{4I_1(x)}{xI_0(x)} < \frac{4}{x} < 1 \quad\text{for all } x>4.
\end{equation*}
Since $h$ is continuous and strictly decreasing on $(0,\infty)$, \Cref{eq:unif:root_equation} has a unique solution
$x_0\in(0,5)$, and hence $\chi_\mathrm{unif}\coloneqq x_0^2/4$ is the unique strictly positive root of $g$.

By \Cref{eq:unif:root_equivalence}, the sign of $g(\chi)$ for $\chi>0$ agrees with the sign of $h(x)-1$ at $x=2\sqrt{\chi}$,
since the remaining factors are strictly positive.
As $h$ is strictly decreasing and crosses $1$ exactly once at $x_0$, we obtain $h(x)>1$ for $x<x_0$ and $h(x)<1$ for $x>x_0$.
Equivalently, $g(\chi)>0$ for $\chi\in(0,\chi_\mathrm{unif})$ and $g(\chi)<0$ for $\chi\in(\chi_\mathrm{unif},\infty)$.

Numerically, $x_0\approx 3.325\,$ and thus $\chi_\mathrm{unif}=x_0^2/4\approx 2.76531$.
\end{proof}
\end{document}